\documentclass{article}

\usepackage{microtype}
\usepackage{graphicx}
\usepackage{subfigure}
\usepackage{booktabs} % for professional tables
\usepackage{amsmath,amsfonts,amsthm}
\usepackage{amssymb}

\usepackage{hyperref}

\newcommand\numeq[1]%
  {\stackrel{\scriptscriptstyle(\mkern-1.5mu#1\mkern-1.5mu)}{=}}
 \newcommand\numgeq[1]%
  {\stackrel{\scriptscriptstyle(\mkern-1.5mu#1\mkern-1.5mu)}{\geq}}

\usepackage[accepted]{icml2021}

\newtheorem{thm}{Theorem}
\newtheorem{lem}[thm]{Lemma}
\newtheorem{cor}[thm]{Corollary}
\usepackage{xcolor}
\usepackage{graphicx}
\DeclareMathOperator*{\argmax}{arg\,max}
\DeclareMathOperator*{\argmin}{arg\,min}
\newtheorem{defn}[thm]{Definition}
\def\D{{\mathcal D}}
\def\Pphi{\overline{\Phi}}
\def\F{{\mathcal F}}
\def\N{{\mathcal N}}
\def\E{{\mathbb E}}
\def\A{\Pi}
\def\B{\Sigma}
\def\diam{\text{diam}}

\def\R{\mathbb{R}}

\def\L{\mathcal L}

\def\H{\mathcal H}
\def\calH{\mathcal H}

\def\R{\mathbb R}
\def\Y{\{\pm 1\}}
\def\U{\mathbb U}
\def\dd{\Delta}

\def\g{g}
\def\rr{R}
\def\f{f}

\newcommand\rbb[1]{\textcolor{purple}{#1}}

\begin{document}

\twocolumn[
\icmltitle{Sample Complexity of Robust Linear Classification on Separated Data}




\begin{icmlauthorlist}
\icmlauthor{Robi Bhattacharjee}{to}
\icmlauthor{Somesh Jha}{goo}
\icmlauthor{Kamalika Chaudhuri}{to}
\end{icmlauthorlist}

\icmlaffiliation{to}{University of California, San Diego}
\icmlaffiliation{goo}{University of Wisconsin-Madison}

\icmlcorrespondingauthor{Robi Bhattacharjee}{rcbhatta@eng.ucsd.edu}

\icmlkeywords{Machine Learning, ICML}

\vskip 0.3in
]



\printAffiliationsAndNotice{\icmlEqualContribution} % otherwise use the standard text.

%\begin{abstract}
%Building classifiers that are robust to adversarial perturbations is a staple of modern machine learning. In most prior work, the set of adversarial perturbations for an instance $x$, $\U_x$, is typically assumed to be a ball centered at $x$ with fixed radius $r$. During training, the robustness radius, $r$, is typically chosen ahead of time as a hyper-parameter. One issue with this approach is that data in different regions of space could have different relevant distance scales. Furthermore, while balls are chosen for their simplicity, it is plausible that different directions could require different robustness radii. Motivated by these issues in this work we consider \textit{data dependent robustness}, where the adversary $\U$ is matched with a data distribution $\D$. We first characterize neighborhood preserving pairs, $(\D, \U)$, which are pairs of distributions and adversaries that satisfy certain natural properties. We then provide sufficient conditions for non-parametric algorithms (including nearest neighbors and kernel classifiers) to converge towards robust solutions for \textit{any} neighborhood preserving pair. Importantly, we show that these algorithms converge even without prior knowledge on what adversary $\U$ is chosen. 
%\end{abstract}

\begin{abstract}
We consider the sample complexity of learning with adversarial robustness. Most prior theoretical results for this problem have considered a setting where different classes in the data are close together or overlapping. We consider, in contrast, the well-separated case where there exists a classifier with perfect accuracy and robustness, and show that the sample complexity narrates an entirely different story. Specifically, for linear classifiers, we show a large class of well-separated distributions where the expected robust loss of any algorithm is at least $\Omega(\frac{d}{n})$, whereas the max margin algorithm has expected standard loss $O(\frac{1}{n})$. This shows a gap in the standard and robust losses that cannot be obtained via prior techniques. Additionally, we present an algorithm that, given an instance where the robustness radius is much smaller than the gap between the classes, gives a solution with expected robust loss is $O(\frac{1}{n})$. This shows that for very well-separated data, convergence rates of $O(\frac{1}{n})$ are achievable, which is not the case otherwise. Our results apply to robustness measured in any $\ell_p$ norm with $p > 1$ (including $p = \infty$).
\end{abstract}

\section{Introduction}

Motivated by the use of machine learning in safety-critical settings, adversarially robust classification has been of much recent interest. Formally, the problem is as follows. A learner is given training data drawn from an underlying distribution $D$, a hypothesis class $\calH$, a robustness metric $d$, and a radius $r$. The learner's goal is to find a classifier $h \in \calH$ which has the lowest robust loss at radius $r$. The robust loss of a classifier is the expected fraction of examples where either $f(x) \neq y$ or where there exists an $x'$ at distance $d(x, x') \leq r$ such that $f(x) \neq f(x')$.  Robust classification thus aims to find a classifier that maximizes accuracy on examples that are distance $r$ or more from the decision boundary, where distances are measured according to the metric $d$.

In this work, we ask: how many samples are needed to learn a classifier with low robust loss when $\calH$ is the class of linear classifiers, and $d$ is an $\ell_p$-metric? Prior work has provided both upper~\cite{bartlett19, ravikumar20} as well as lower bounds~\cite{schmidt18, ravikumar20} on the sample complexity of the problem. However, almost all look at settings where the data distribution itself is not separated --  data from different classes overlap or are close together in space. In this case, the classifier that minimizes robust loss is quite different from the one that minimizes error, which often leads to strong sample complexity gaps. Many real tasks where robust solutions are desired however tend to involve well-separated data~\cite{yang2020}, and hence it is instructive to look at what happens in these cases.

With this motivation, we consider in this work robust classification of data that is linearly $r$-separable. Specifically, there exists a linear classifier which has zero robust loss at robustness radius $r$. This case is thus the analog of the realizable case for robust classification, and we consider both upper and lower bounds in this setting.

For lower bounds, prior work \cite{nips18_lame} shows that both standard and robust linear classification have VC-dimension $O(d)$, and consequently have similar bounds on the expected loss in the worst case. However, these results do not apply to this setting since we are specifically considering well-separated data, which greatly restricts the set of possible worst-case distributions.  For our lower bound, we provide a family of distributions that are linearly $r$-separable and where the maximum margin classifier, given $n$ independent samples, has error $O(1/n)$. In contrast, any algorithm for finding the minimum robust loss classifier has robust loss at least $\Omega(d/n)$, where $d$ is the data dimension. These bounds hold for all $\ell_p$-norms provided $p > 1$, including $p=2$ and $p=\infty$. Unlike prior work, our bounds do not rely on the difference in loss between the solutions with optimal robust loss and error, and hence cannot be obtained by prior techniques. Instead, we introduce a new geometric construction that exploits the fact that learning a classifier with low robust loss when data is linearly $r$-separated requires seeing a certain number of samples close to the margin.

For upper bounds, prior work \cite{bartlett19} provides a bound on the Rademacher complexity of adversarially robust learning, and show that it can be worse than the standard Rademacher complexity by a factor of $d^{1/q}$ for $\ell_p$-norm robustness where $1/p + 1/q = 1$. Thus, an interesting question is whether dimension-independent bounds, such as those for the accuracy under large margin classification, can be obtained for robust classification as well. Perhaps surprisingly, we show that when data is really well-separated, the answer is yes. Specifically, if the data distribution is linearly $r + \gamma$-separable, then there exists an algorithm that will find a classifier with robust loss $O(\Delta^2/\gamma^2 n)$ at radius $r$ where $\Delta$ is the diameter of the instance space. Observe that much like the usual sample complexity results on SVM and perceptron, this upper bound is independent of the data dimension and depends only on the excess margin (over $r$). This establishes that when data is really well-separated, finding robust linear classifiers does not require a very large number of samples. 

While the main focus of this work is on linear classifiers, we also show how to generalize our upper bounds to Kernel Classification, where we find a similar dynamic with the loss being governed by the excess margin in the embedded kernel space. However, we defer a thorough investigation of robust kernel classification as an avenue for future work.

Our results imply that while adversarially robust classification may be more challenging than simply accurate classification when the classes overlap, the story is different when data is well-separated. Specifically, when data is linearly (exactly) $r$-separable, finding an $r$-separated solution to robust loss $\epsilon$ may require $\Omega(d/\epsilon)$ samples for some distribution families where finding an accurate solution is easier. Thus in this case, there is a gap between the sample complexities of robust and simply accurate solutions, and this is true regardless of the $\ell_p$ norm in which robustness is measured. In contrast, if data is even more separated -- linearly $r + \gamma$-separable --  then we can obtain a dimension-independent upper bound on the sample complexity, much like the sample complexity of SVMs and perceptron. Thus, how separable the data is matters for adversarially robust classification, and future works in the area should consider separability while discussing the sample complexity

\subsection{Related Work}

There is a large body of work \cite{Carlini17, Liu17, Papernot17, Papernot16, Szegedy14, Hein17, Katz17, Wu16,Steinhardt18, Sinha18} empirically studying adversarial examples primarily in the context of neural networks. Several works \cite{schmidt18, Raghunathan20, Tsipras19} have empirically investigated trade-offs between robust and standard classification.

On the theoretical side, this phenomenon has been studied in both the parametric and non-parametric settings. On the parametric side, several works \cite{loh18, attias19, Srebro19, bartlett19, pathak20} have focused on finding distribution agnostic bounds of the sample complexity for robust classification. In \cite{Srebro19}, Srebro et. al. showed through an example that the VC dimension of robust learning may be much larger than standard or accurate learning indicating that the sample complexity bounds may be higher. However, their example did not apply to linear classifiers. 

\cite{Kane20} considers learning linear classifiers robustly, but is primarily focused on computational complexity as opposed to sample complexity.

In \cite{bartlett19}, Bartlett et. al. investigated the Rademacher complexity of robustly learning linear classifiers as well as neural networks. They showed that in both cases, the robust Rademacher complexity can be bounded in terms of the dimension of the input space -- thus indicating a possible gap between standard and robust learning. However, as with the works considering VC dimension, this work is fundamentally focused on upper bounds  -- they do not show true lower bounds on data requirements.

Because of it's simplicity and elegance, the case where the data distribution is a mixture of Gaussians has been particularly well-studied. The first such work was \cite{schmidt18}, in which Schmidt et. al. showed an $\Omega(\sqrt{d})$ gap between the standard and robust sample complexity for a mixture of two Gaussians using the $\ell_\infty$ norm. This was subsequently expanded upon in \cite{Bhagoji19}, \cite{robey20} and  \cite{ravikumar20}. \cite{Bhagoji19} introduces a notion of ``optimal transport," which they subsequently apply to the Gaussian case, deriving a closed form expression for the optimally robust linear classifier. Their results apply to any $\ell_p$ norm. \cite{robey20} applies expands upon \cite{schmidt18} by consider mixtures of three Gaussians in both the $\ell_2$ and $\ell_\infty$ norms. Finally, \cite{ravikumar20} fully generalizes the results of \cite{schmidt18} providing tight upper and lower bounds on the standard and robust sample complexities of a mixture of two Gaussians, in any norm (including $\ell_p$ for $p \in [1, \infty]$). \cite{schmidt18} and \cite{ravikumar20} bear the most relevance with our work, and we consequently carefully compare our results in section \ref{sec:comparison}.

Another approach for lower and upper bounds on sample complexities for linear classifiers can be found in \cite{nips18_lame}, which examines the robust VC dimension of learning linear classifiers. They show that the VC dimension is $d+1$, just as it is in the standard case. This implies that the bounds in the robust case match the bounds in the standard case and in particular shows a lower bound of $\Omega(d/n)$ on the expected loss of learning a robust linear classifier from $n$ samples.

While this result appears to match our lower bound, there is a crucial distinction between the bounds. Our bound implies that there exists some distribution with a large $\ell_2$ margin for which the expected robust loss must be $\Omega(d/n)$. On the other hand, standard results about learning linear classifiers on large margin data implies that the expected standard loss will be $O(1/n)$ (when running the max-margin algorithm). For this reason, our paper provides a case in the well-separated setting in which learning linear classifiers is provably more difficult (in terms of sample complexity) in the robust setting than in the standard setting. By contrast, \cite{nips18_lame} does not show this. Their paper only implies (through standard VC constructions) the existence of \textit{some} distribution that is difficult to learn, and the standard PAC bounds cannot ensure that such a distribution also has a large $\ell_2$ margin.

In the non-parametric setting, there are several works which contrast standard learning with robust learning. \cite{WJC18} considers the nearest neighbors algorithm, and shows how to adapt it for converging towards a robust classifier. In \cite{YRWC19}, Yang et. al. propose the $r$\textit{-optimal classifier}, which is the robust analog of the Bayes optimal classifier. Through several examples they show that it is often a fundamentally different classifier - which can lead to different convergence behavior in the standard and robust settings. \cite{Bhattacharjee20} unified these approaches by specifying conditions under which non-parametric algorithms can be adapted to converge towards the $r$-optimal classifier, thus introducing $r$-consistency, the robust analog of consistency. 

\section{Preliminaries}
We consider binary classification over $\R^d \times \Y$. Our metric of choice is the $\ell_p$ norm, where $p > 1$ (including $p = \infty$) is arbitrary. For $x \in \R^d$, we will use $||x||_p$ to denote the $\ell_p$ norm of $x$, and consequently will use $||x - y ||_p$ to denote the $\ell_p$ distance between $x$ and $y$. We will also let $\ell_q$ denote the dual norm to $\ell_p$ - that is, $\frac{1}{q} + \frac{1}{p}= 1$.

 We use $B_p(x,r)$ to denote the closed $\ell_p$ ball with center $x$ and radius $r$. For any $S \subset \R^d$, we let $diam_p(S)$ denote its diameter: that is, $diam_p(S) = \sup_{x, y \in S} ||x - y||_p.$

\subsection{Standard and Robust Loss}

In classical statistical learning, the goal is to learn an accurate classifier, which is defined as follows:

\begin{defn}
Let $\D$ be a distribution over $\R^d \times \Y$, and let $f \in \Y^{\R^d}$ be a classifier. Then the \textbf{standard loss} of $f$ over $\D$, denoted $\L(f, \D)$, is the fraction of examples $(x,y) \sim \D$ for which $f$ is not accurate. Thus $$\L(f, \D) = P_{(x,y) \sim \D}[f(x) \neq y].$$
\end{defn}

Next, we define robustness, and the corresponding robust loss.

\begin{defn}
A classifier $f \in \Y^{\R^d}$ is said to be \textbf{robust} at $x$ with radius $r$ if $f(x) = f(x')$ for all $x' \in B_p(x,r)$. 
\end{defn}

\begin{defn}
The \textbf{robust loss} of $f$ over $\D$, denoted $\L_r(f, \D)$, is the fraction of examples $(x,y) \sim \D$ for which $f$ is either inaccurate at $(x,y)$, or $f$ is not robust at $(x,y)$ with radius $r$. Observe that this occurs if and only if there is some $x' \in B_p(x,r)$ such that $f(x') \neq y$. Thus $$\L_r(f, \D) = P_{(x, y) \sim \D}[\exists x' \in B_p(x,r)\text{ s.t. }f(x') \neq y].$$ 
\end{defn}

\subsection{Expected Loss and  Sample Complexity}

The most common way to characterize the performance of a learning algorithm is through an $(\epsilon, \delta)$ guarantee, which computes $\epsilon_n, \delta_n$ such that an algorithm trained over $n$ samples has loss at most $\epsilon_n$ with probability at least $1 - \delta_n$. 

In this work, we use the simpler notion of \textit{expected loss}, which is defined as follows: 

\begin{defn}
Let $A$ be a learning algorithm and let $\D$ be a distribution over $\R^d \times \{\pm 1\}$. For any $S \sim \D^n$, we let $A_S$ denote the classifier learned by $A$ from training data $S$. Then the \textbf{expected standard loss} of $A$ with respect to $\D$, denoted $EL^n(A, \D)$ where $n$ is the number of training samples, is defined as $$E\L^n(A, \D) = \E_{S \sim \D^n} \L(A_S, \D).$$ Similarly, we define the \textbf{expected robust loss} of $A$ with respect to $\D$ as $$E\L_r^n(A , \D) = \E_{S \sim \D^n}\L_r(A_S, \D).$$ 
\end{defn}

Our main motivation for using this criteria is simplicity. Our primary goal is to compare and contrast the performances of algorithms in the standard and robust cases, and this contrast clearest when the performances are summarized as a single number (namely the expected loss) rather than an $(\epsilon, \delta)$ pair. 

Next, we address the notion of sample complexity. As above, sample complexity is typically defined as the minimum number of samples needed to guarantee $(\epsilon, \delta)$ performance. In this work, we will instead define it solely with respect to $\epsilon$, the expected loss. 

\begin{defn}
Let $\D$ be a distribution over $\R^d \times \{\pm 1\}$ and $A$ be a learning algorithm. Then the \textbf{standard sample complexity} of $A$ with respect to $\D$, denoted $m^\epsilon(A, \D)$, is the minimum number of training samples needed such that $A$ has  expected standard loss at most $\epsilon$. Formally, $$m^\epsilon(A, \D) = \min(\{n: E\L^n(A, D) \leq \epsilon\}).$$ Similarly, we can define the \textbf{robust sample complexity} as $$m_r^\epsilon(A, \D) = \min(\{n: E\L^n(A, D) \leq \epsilon\}).$$
\end{defn}

\subsection{Linear classifiers}

In this work, we consider linear classifiers, formally defined as follows:
\begin{defn}
Let $w \in \R^d$ be a vector. Then the \textbf{linear classifier} with parameters $w \in \R^d$ and $b \in \R$ over $\R^d \times {\pm 1}$, denoted $f_{w, b}$, is defined as , $$f_{w,b}(x) = \begin{cases} +1 & \langle w, x \rangle  \geq b \\ -1 & \langle w, x \rangle < b  \end{cases}.$$ 
\end{defn}

Learning linear classifiers is well understood in the standard classification setting. We now consider the linearly \textit{separable} case, in which some linear classifier has perfect accuracy. We will later define linear $r$-separability as the robust analog of separability.

\begin{defn}
A distribution $\D$ over $\R^d \times Y$ is \textbf{linearly separable} if its support can be partitioned into sets $S^+$ and $S^-$ such that:

1. $S^+$ and $S^-$ correspond to the positively and negatively labeled subsets of $\R^d$. In particular, $P_{(x,y) \sim \D}[x \in S^y] = 1.$

2. There exists a linear classifier, $f_{w, b}$, that has perfect accuracy. That is, $\L(f_{w, b}, \D) = 0$. 
\end{defn}

The standard sample complexity for linearly separable distributions can be characterized through their margin, which is defined as follows.

\begin{defn}\label{defn:margin}
Let $\D$ be a linearly separable distribution over $\R^d \times \{\pm 1\}$. Let $S^+$ and $S^-$ be as above. Then $\D$ has \textbf{margin} $\gamma$ if $\gamma$ is the largest real number such that there exists a linear classifier $f_{w,b}$ with the following properties:

1. $f_{w,b}$ has perfect accuracy. That is, $\L(f_{w,b}, \D) = 0$.

2. Let $H_{w,b} = \{x: \langle x, w \rangle = b\}$ denote the decision boundary of $f_{w,b}$. Then for all $x \in (S^+ \cup S^-)$, $x$ has $\ell_2$ distance at least $\gamma$ from $H_{w,b}$. That is, $$\inf_{x \in S^+ \cup S^-, z \in H_{w,b}} ||x - z||_2 \geq \gamma.$$ We let $\gamma(\D)$ denote the margin of $\D$.
\end{defn}

Observe that although we use a general norm, $\ell_p$, to measure robustness, the margin is always measured in $\ell_2$. This is because the $\ell_2$ norm plays a fundamental role in  bounding the number of samples needed to learn a linear classifier. 

The basic idea is that when the $\ell_2$ margin is large relative to the $\ell_2$ diameter of the distribution, the max margin algorithm requires fewer samples needed to learn a linear classifier. In particular, the ratio between the $\ell_2$ margin and the $\ell_2$ diameter fully characterizes the standard sample complexity of the max margin algorithm. To further simplify our notation, we define this ratio as the aspect ratio.

\begin{defn}\label{defn:aspect_ratio}
Let $\D$ be a linearly separable distribution over $\R^d \times \{\pm 1\}$. Then the \textbf{aspect ratio} of $\D$, $\rho(\D)$ is defined as, $$\rho(\D) = \frac{diam_2(S^+ \cup S^-)}{\gamma(\D)},$$ where $\diam_2(S^+ \cup S^-)$ denotes its diameter in the $\ell_2$ norm.
\end{defn}

We now have the following well-known result, which characterizes the expected standard loss with the aspect ratio.
\begin{thm}\label{thm:standard}
\emph{(Chapter 10 in \cite{vapnik1998})} Let $M$ denote the hard margin SVM algorithm. If $\D$ is a distribution with aspect ratio $\rho = \rho(\D)$, then for any $n > 0$ we have $\E_{S \sim \D^n}\L(M_S, \D) \leq O(\frac{\rho^2}{n}),$ where $M_S$ denotes the classifier learned by $M$ from training data $S$. 
\end{thm}

We can also express this result in terms of standard sample complexity.
\begin{cor}\label{cor:standard}
Let $M$ denote the hard margin SVM algorithm. If $\D$ is a distribution with aspect ratio $\rho = \rho(\D)$, then for any $\epsilon > 0$ we have $m^\epsilon(M_S, \D) \leq O(\frac{\rho^2}{\epsilon}),$ where $M_S$ denotes the classifier learned by $M$ from training data $S$. 
\end{cor}

Theorem \ref{thm:standard} and Corollary \ref{cor:standard} will serve as a benchmark for comparison with the robust sample complexity.  
\subsection{Linear $r$-separability}

Finally, we introduce linear $r$-separability, which is the key characteristic of distributions considered in this paper. This can be thought of as the robust analog of linear separability.
\begin{defn}\label{defn:r_separability}
For any $r > 0$, a distribution $\D$ over $\R^d \times \{\pm 1\}$ is \textbf{linearly} $r$-\textbf{separable} if there exists a linear classifier $f_{w, b}$ such that $\L_r(f_{w, b}, \D) = 0$.
\end{defn}
This definition is the fundamental property considered in this paper. Our goal is to understand the sample complexity required for learning robust linear classifiers on linearly $r$-separable distributions, and compare it with the standard sample complexity given in Theorem \ref{thm:standard}.

\section{Lower Bounds}\label{sec:lower_bounds}

In this section, we consider $r$-separated distributions whose aspect ratio is constant. By Theorem \ref{thm:standard}, the standard sample complexity for learning them is independent of $d$. We will show that in contrast, the robust sample complexity has a linear dependence on $d$, and consequently establish a substantial gap between the standard and robust cases.

We begin by defining the family of such distributions.
\begin{defn}
For any $\rho, r$, the set $\F_{r, \rho}$ is defined as the set of all distributions $\D$ over $\R^d \times \{\pm 1\}$ such that $\D$ is $r$-separated and has aspect ratio at most $\rho$.
\end{defn} 

We now state our main result.
\begin{thm}\label{thm:lower}
Let $r > 0$ and $\rho > 20$. Then the following hold.
\begin{enumerate}
	\item For every learning algorithm $A$, and any $n > 0$, there exists $\D \in \F_{r, \rho}$ such that the expected robust loss when $A$ is trained on a sample of size $n$ from $\D$ is at least $\Omega(\frac{d}{n})$. Formally, there exists a constant $c > 0$ such that $\E_{S \sim \D^n}[\L_r(A_S, \D)] \geq \frac{cd}{n}.$
	\item  In contrast, by Theorem \ref{thm:standard}, for \textit{any} $\D \in \F_{r, D}$, the max margin algorithm has expected standard loss $O(\frac{\rho^2}{n})$, when trained on a sample of size $n$ from $\D$. Formally, there exists a constant $c' > 0$ such that $\E_{S \sim \D^n}[\L(A_S, \D)] \leq \frac{c'\rho^2}{n}.$
\end{enumerate}
\end{thm}
The condition $\rho > 20$ is required to rule out degenerate cases. This is because for small values of $\rho$, the $\ell_2$ diameter of $\D$ is not much larger than the $\ell_2$ margin of $\D$. This forces $\D$ to be mostly clustered around a line which leads to more complicated behavior.

Observe that when $\rho$ is a constant independent of $d$, the expected  standard loss is $O(\frac{1}{n})$ while the expected robust loss is $\Omega(\frac{d}{n})$. Thus, the ratio between the expected robust loss and the expected standard loss is $\Omega(d)$, leading to a dimensional dependent gap between the robust and standard cases. 

We also note that these bounds hold regardless of which $\ell_p$ ($p \in (1, \infty])$ norm is being used. This is because our construction of $\D \in \F_{r, \rho}$ for which the lower bound holds is given in terms of the norm $p$. More generally, the family $\F_{r, \rho}$ is implicitly defined with respect to $p$.

Furthermore, our lower bound differs from the lower bound of $\Omega(\frac{d}{n})$ shown in prior work \cite{nips18_lame} because it specifically holds for $\F_{r, \rho}$, a linearly $r$-separated family of distributions with constant aspect ratio. Thus, while \cite{nips18_lame} has shown the existence of distributions satisfying the first condition of Theorem \ref{thm:lower}, our result is the first to exhibit a distribution satisfying both conditions.

Finally, we note that Theorem \ref{thm:lower} can also be expressed in terms of sample complexities. We include this in the following corollary.

\begin{cor}
Let $r > 0$ and $\rho > 20$. Then the following hold.

1. For every learning algorithm $A$, and any $\epsilon > 0$, there exists $\D \in \F_{r, \rho}$ such that the robust sample complexity of $A$ with respect to $\D$ is at least $\Omega(\frac{d}{\epsilon})$. Formally, there exists a constant $c > 0$ such that $m_r^\epsilon(A, \D) \geq \frac{cd}{\epsilon}.$

2. In contrast, by Theorem \ref{thm:standard}, for \textit{any} $\D \in \F_{r, D}$, the max margin algorithm has standard sample complexity $O(\frac{\rho^2}{\epsilon})$. Formally, there exists a constant $c' > 0$ such that $m^\epsilon(A, \D) \leq \frac{c'\rho^2}{\epsilon}.$

\end{cor}

\subsection{Comparison with \cite{ravikumar20} and~\cite{schmidt18}}\label{sec:comparison}

The first work to provide a robust sample complexity lower bound that applied to linear classifiers is~\cite{schmidt18}; they showed a gap of $\Omega(\sqrt{d})$ between the robust and accuracy loss for a specific mixture of two Gaussians. This was later generalized to mixtures of any two Gaussians by~\cite{ravikumar20}, who also established more general lower bounds for any $\ell_p$ norm. Since \cite{ravikumar20} is a strict generalization of \cite{schmidt18}, we next explain how our lower bounds differ from~\cite{ravikumar20}, and why their techniques do not lead to our results. We begin by summarizing their results.

\paragraph{Summary of \cite{ravikumar20}} \cite{ravikumar20} considers data distributions $\D$ that are parametrized by $\mu \in \R^d$ and $\Sigma \in \R^{d \times d}$, $\Sigma \succcurlyeq 0$. $\D_{\mu, \Sigma}$ is the mixture of two Gaussians, $\N(\mu, \Sigma)$ and $\N(-\mu, \Sigma)$, with equal mass, where instances drawn from $\N(\mu, \Sigma)$ are labeled as $+$, and instances drawn from $\N(-\mu, \Sigma)$ are labeled as $-$. They consider robustness measured in any normed metric in $\R^d$, including the $\ell_p$ norm for $p \in (1, \infty]$. Although their bounds apply to any classifier, this effectively deals with linear classifiers since it can be shown that the optimally robust and accurate classifiers are both linear.

For any distribution $\D_{\mu, \Sigma}$, let $L_{rob}$ denote the optimal robust loss of any classifier on $\D_{\mu,\Sigma}$, and let $L_{std}$ denote the optimal standard loss. Then the bounds shown in \cite{ravikumar20} can restated as follows (a detailed derivation from \cite{ravikumar20} appears in Appendix \ref{sec:appendix_comparison}). 

\begin{thm}\label{thm:ravikumar}
\cite{ravikumar20}
\begin{enumerate}
	\item For any learning algorithm $A$ and any $n > 0$, there exists some mixture of Gaussians, $\D_{\mu, \Sigma}$ such that the expected \textit{excess} robust loss is at least $\Omega(L_{rob}\frac{d}{n}),$ when $A$ is trained on a sample of size $n$ from $\D$. 
	\item For any distribution $\D_{\mu, \Sigma}$, it is possible to learn a classifier with expected \textit{excess} standard loss at most $O(L_{std}\frac{d}{n})$.
	\item By (1.) and (2.), the ratio between the expected excess loss and expected excess standard loss can be expressed as $ratio \geq \Omega(\frac{L_{rob}}{L_{std}}).$
\end{enumerate}
\end{thm}

Observe that their bounds are given through \textit{excess} losses, which is the amount by which the loss exceeds to the optimal loss. This is necessary because in their setting, the optimal classifiers do not have $0$ loss. 

\paragraph{Comparison with our bounds} Recall that in our work, we are concerned with the \textit{linearly $r$-separated case}, which occurs precisely when the optimal robust and standard losses both equal $0$. However, from Theorem \ref{thm:ravikumar}, we see that although \cite{ravikumar20} proves a gap between standard and robust sample complexity, this gap is predicated on distributions for which the optimal robust loss, $L_{rob}$ and optimal standard loss, $L_{std}$ differ. Furthermore, in the case where they obtain a gap of $\Omega(d)$, we see that this requires $\frac{L_{rob}}{L_{std}} = \Omega(d)$ which is a substantial difference. By contrast, our results characterize a gap exclusively in the case that this does not occur. 

Finally, in the limiting case where the Gaussians they consider are sufficiently far apart, their data will begin to appear linearly $r$-separated, meaning both $L_{rob}$ and $L_{std}$ are close to $0$. However, even in this case, it can be shown that the ratio $\frac{L_{rob}}{L_{std}}$ diverges towards infinity, meaning that their lower bound characterizes a very different dynamic from ours. Precise details on this comparison can be found in appendix \ref{sec:appendix_comparison}.

\subsection{Intuition behind Theorem \ref{thm:lower}}

The proof idea for Theorem \ref{thm:lower} can be summarized with a simple example (Figure \ref{fig:small_margin_robust}). In this example, we seek to learn a linear classifier for a linearly $r$-separated distribution in $\R^2$. The key idea is to contrast the necessary conditions for learning a robust classifier, and the necessary conditions for learning an accurate classifier. 

Observe that the distribution is \textit{precisely} linearly $r$-separated, that is, it is not possible to achieve robustness for radii larger than $r$. Because of this, there is a unique linear classifier $f_{rob}$ that has perfect robustness. In order to learn this classifier, we must see examples from $S^+ \cup S^-$ that are close to the ``boundary" of $S^+ \cup S^-$. In our figure, this consists of points that are close to the dotted blue and red lines. Moreover, it can be shown that the number of such examples we must see is related to $d$, the dimension.

By contrast, any classifier that separates $S^+$ from $S^-$ has perfect accuracy (take for example $f_{std}$ shown in the figure). It is possible to exploit this by using margin based algorithms for learning linear classifiers. In particular, we no longer need to see points that are extremely close to the boundary of $S^+ \cup S^-$. 
 
\begin{figure}[h]
\vspace{.3in}
\includegraphics[scale=0.55]{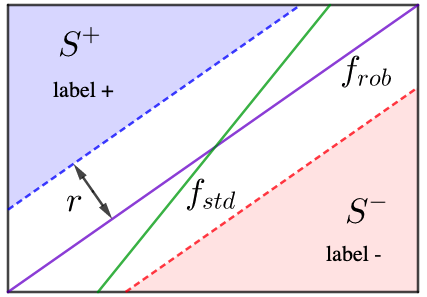}
\vspace{.3in}
\caption{An example of a linearly $r$-separated distribution, with positively and negatively labeled examples in $S^+$ and $S^-$ respectively. The optimally robust classifier, $f_{rob}$ is shown in purple, while the (not necessarily unique) optimally accurate classifier, $f_{std}$, is shown in green.}
\label{fig:small_margin_robust}
\end{figure}

\paragraph{General Hypothesis Classes:} We now briefly consider how to extend our methods to other hypothesis classes. For any hypothesis class $\H$ and distribution $\D$ let $$\H_{\D, \alpha} = \{h: h \in \H,\L(h, \D) \leq \alpha\}$$ and let $$\H_{\D, \alpha}^r = \{h: h \in \H,\L_r(h, \D) \leq \alpha\}.$$ $\H_{\D, \alpha}$ can be thought of as the set of accurate classifiers while $\H_{\D, \alpha}^r$ can be thought of as the set of astute classifiers. By their definitions, it is clear that $\H_{\D, \alpha}^r \subseteq \H_{\D, \alpha}$. However, in the case when $\H$ is the set of linear classifiers, we see that for small $\alpha$, $\H_{\D, \alpha}^r$ is a much ``smaller" set than $\H_{\D, \alpha}$. By exploiting the geometric structure inherent to $\H$, we can much more efficiently search for some $h \in \H_{\D, \alpha}$ than we can in $\H_{\D, \alpha}^r$. This dynamic is the crux of our lower bound: as we essentially show that there are far more critical points (i.e. points near the decision boundary) that we must see for learning $\H_{\D, \alpha}^r$ that aren't required for $\H_{\D, \alpha}$. 

Thus, for our methods to extend to an arbitrary hypothesis class, we would require a similar dynamic. We need two properties to hold: (1) $\H_{\D, \alpha}^r$ must be a very strict subset of $\H_{\D, \alpha}$ for sufficiently small alpha. (2) We must have some kind of exploitable geometric structure about $\H$ which allows us to exploit this gap. For the case of linear classifiers, this was the $\ell_2$ measured aspect ratio, $\gamma(\D)$. 

\paragraph{Kernel Classifiers: } A natural choice of a more general hypothesis class would be Kernel Classifiers, which are linear classifiers that operate in an embedded space, $H$. The main difficulty in expanding our lower bound to this more general setting comes from the behavior near the margin: the effects of the robustness radius in the embedded space are considerably less behaved than they are in the standard linear case. Nevertheless, we leave this as an important avenue for future work.

\section{Upper Bounds}\label{sec:upper_bound}

\begin{figure}
\begin{algorithm}[H]
   \caption{Adversarial-Perceptron}
   \label{alg:upper_bound}
\begin{algorithmic}[1]
    \STATE \textbf{Input}:  $S = \{(x_1, y_1), \dots, (x_n, y_n)\} \sim \D^n,$
    \STATE $w \leftarrow 0$ 
    \FOR{$i = 1 \dots n$}
    	\STATE $z = \argmin_{||z - x_i||_p \leq r}  y_i\langle w, z \rangle$ \COMMENT{\rbb{finds adv. ex.}}
        \IF{$\langle w, y_iz \rangle \leq 0$ \COMMENT{\rbb{checks label}}}
            \STATE $w \leftarrow w + y_iz$ \COMMENT{\rbb{perceptron update}}
        \ENDIF           
    \ENDFOR
    \STATE return $f_{w, 0}$
\end{algorithmic}
\end{algorithm}
\caption{An algorithm combining adversarial training with the perceptron algorithm. For each $(x_i, y_i)$, we first attack it, to get $z$. If $z$ is labeled incorrectly, we do a perceptron update using $z$.}
\end{figure}

In the previous section, we showed that for any algorithm, there is some distribution $\D \in \F_{r, \rho}$ that is difficult (i.e. requires high sample complexity) to learn robustly. A natural follow-up question is: what about distributions for which the margin, $\gamma$ is very large compared to $r$. 

Observe that in Figure \ref{fig:small_margin_robust} the robustness radius $r$ is very close to the margin. In particular, we can find adversarial examples from $S^+$ and $S^-$ that are very close to the decision boundary $f_{rob}$. By contrast, if $\gamma >> r$, then this no longer holds which suggests that better robust sample complexities might be possible.

In this section, we will describe a subset of $\F_{r, \rho}$ that can be learned with expected loss $O(\frac{1}{n})$, thus matching the standard sample complexity up to a constant factor. To do so, we will introduce a novel concept: the \textit{robust margin}. The basic intuition is that distributions for which the margin greatly exceeds the robustness radius are precisely distributions with a large robust margin. We use the following notation.

Observe that if $\D$ is a linearly $r$-separated distribution, then $\D$ must also be linearly separable. As earlier, let $S^+, S^- \subset \R^d$ denote the positively and negatively labeled examples from $\D$. We now define \begin{equation}\label{eqn:s_plus_s_minus} S_r^+ = \cup_{s \in S^+} B_p(s, r)\text{ and }S_r^{-} = \cup_{s \in S^-} B_p(s,r).\end{equation} It follows that the decision boundary of any linear classifier with perfect robustness over $\D$ must separate $S_r^+$ and $S_r^-$. We now define the robust margin as a measurement of this separation.

\begin{defn}\label{def:robust_margin}
Let $\D$ be a linearly $r$-separable distribution over $\R^d \times \{\pm 1\}$. Let $S_r^+$ and $S_r^-$ be as above. Then $\D$ has \textbf{robust margin} $\gamma_r$ if $\gamma_r$ is the largest real number such that there exists a linear classifier $f_{w,b}$ with the following properties: 

1. $f_{w,b}$ has perfect astuteness. That is, $\L_r(f_{w,b}, \D) = 0$. 

2. Let $H_{w,b} = \{x: \langle x, w \rangle = b\}$ denote the decision boundary of $f_{w,b}$. Then for all $x \in (S_r^+ \cup S_r^-)$, $x$ has $\ell_2$ distance at least $\gamma$ from $H_{w,b}$. That is, $$\inf_{x \in S_r^+ \cup S_r^-} \inf_{z \in H_{w,b}} ||x - z||_2 \geq \gamma.$$ We let $\gamma_r(\D)$ denote the margin of $\D$, and say that such a distribution is $r, \gamma_r$-separated. 
\end{defn}

It is crucial to note that although adversarial perturbations are measured in $\ell_p$, the robust margin is measured in $\ell_2$. This is because while the metric $\ell_p$ plays a role in constructing $B(x,r)$, it can be completely disregarded once the sets $S_r^+$ and $S_r^-$ are considered, as any hyperplane separating $S_r^+$ and $S_r^-$ will have perfect robustness.  

We now define the robust aspect ratio, which is the robust analog of standard aspect ratio.

\begin{defn}
Let $\D$ be a distribution over $\R^d \times \{\pm 1\}$. Then the \textbf{robust aspect ratio} of $\D$, $\rho_r(\D)$ is defined as $$\rho_r(\D) = \frac{diam_2(S_r^+ \cup S_r^-)}{\gamma_r(\D)},$$ where as before, $diam_2(S_r^+ \cup S_r^-)$ denotes its diameter in the $\ell_2$ norm.
\end{defn}

We will now show that just as the aspect ratio, $\rho(\D)$, characterized the sample complexity for standard classification, the robust aspect ratio, $\rho_r(\D)$ will characterize the sample complexity for robust learning. To do so, we present a perceptron-inspired algorithm (Algorithm \ref{alg:upper_bound}) for learning a robust classifier on $r$-separated data with robust aspect ratio $\rho_r$. 

The basic idea behind Algorithm \ref{alg:upper_bound} is to combine the standard perceptron algorithm with adversarial training. In particular, we iterate through the training set and do the following on each point (refer to Algorithm \ref{alg:upper_bound} for precise details). 

1. Find an adversarial example $(z, y_i)$ by attacking our classifier, $f_{w, 0}$, at $(x_i, y_i)$ \rbb{(line 4)}. This is a straightforward convex optimization problem for linear classifiers.

2. If $f_{w,0}(z) \neq y_i$, we update our weight vector with $(z, y_i)$ by using the standard perceptron update \rbb{(lines 5-6)}.

We have the following upper bound on the expected robust loss of our algorithm.

\begin{thm}\label{thm:upper_bound}
Let $\D$ be a distribution with robust aspect ratio $\rho_r(\D)$. Then for any $n > 0$, we have $$\E_{S \sim \D^n} [\L_r(A_S, \D)] \leq O(\frac{\rho_r(\D)^2}{n}),$$ where $A_S$ denotes the classifier learned by Algorithm \ref{alg:upper_bound} from training data $S$. 
\end{thm}

Observe that this expected loss is still larger than the expected standard loss in  Theorem \ref{thm:standard} as $\rho_r(\D) > \rho(\D)$ for any $\D$. We also note that this result is not contradictory with our lower bound; there exist distributions $\D \in \F_{r, \rho}$ such that $\gamma_r(\D) = 0$, and these are precisely the distributions for which our lower bounds hold. 

\subsection{Generalization to Kernel Classifiers}

Algorithm \ref{alg:upper_bound} can be thought of as the robust analog to the perceptron algorithm. We now generalize this algorithm to obtain a robust variant of the \textit{kernel perceptron algorithm}. We first briefly review kernel classifiers. A detailed explanation of our generalized algorithm along with requisite background material can be found in Appendix \ref{sec:kernel_appendix}

\begin{defn}\label{defn:kernel}
Let $K: \R^d \times \R^d \to \R$ be a kernel similarity function, $T = \{(x_1, y_1), \dots, (x_m, y_m)\} \subset \R^d \times \{\pm 1\}$ be a set of labeled points, and $\alpha \in \R^m$ be a vector of $m$ real numbers. Then the \textbf{kernel classifier} with similarity function $K$, parameters $T, \alpha$, and denoted by $f_{T, K}^\alpha$ is defined as $$f_{T, \alpha}^K(x) = \begin{cases} +1 &  \sum_1^m \alpha_iy_iK(x_i, x) \geq 0\\ -1 &  \sum_1^m \alpha_iy_iK(x_i, x) < 0  \end{cases}.$$
\end{defn}

Conceptually, kernel classifiers are linear classifiers operating in embedded space. With each kernel similarity function $K$, there is a map $\phi: \R^d \to H$ (where $H$ is some Hilbert space) such that $K(x, x') = \langle \phi(x), \phi(x') \rangle$. Thus we can think of kernel classifiers as having a linear decision boundary in $H$. 

We now present an analog of Algorithm \ref{alg:upper_bound} that we call the Adversarial Kernel-Perceptron. The essence of this algorithm has not changed. For each $(x_t, y_t)$ in our training set, we do the following.

1. Find an adversarial example $(z, y_i)$ by attacking our classifier, $f_{T, \alpha}^K$, at $(x_i, y_i)$ \rbb{(line 4)}. 

2. If $f_{T, \alpha}^K(z) \neq y_i$, we update our weight vector with $(z, y_i)$ by appending $(z, y_i)$ to $T$ \rbb{(lines 5-6)}. This corresponds to a kernel-perceptron update that uses $(z, y_i)$ instead of $(x_i, y_i)$.  

\begin{figure}
\begin{algorithm}[H]
   \caption{Adversarial-Kernel-Perceptron}
   \label{alg:upper_bound_kernel}
\begin{algorithmic}[1]
    \STATE \textbf{Input}:  $S = \{(x_1, y_1), \dots, (x_n, y_n)\} \sim \D^n,$ Similarity function, $K$
    \STATE $T \leftarrow \emptyset$, $\alpha \leftarrow 0$
    \FOR{$i = 1 \dots n$}
    	\STATE $z = \argmin_{||z - x||_p \leq r}  y_if_{T, \alpha}^K(z)$ \COMMENT{\rbb{finds adv. ex.}}
        \IF{$f_{T, \alpha}^k(z) \leq 0$\COMMENT{\rbb{checks label}}}
            \STATE $T = T \cup \{(z, y_i)\}$ \COMMENT{\rbb{kern. percep. update}}
            \STATE $\alpha = (1, \dots, 1)_{|T|}$
        \ENDIF           
    \ENDFOR
    \STATE return $f_{T, \alpha}^K$
\end{algorithmic}
\end{algorithm}
\caption{A kernel version of Algorithm \ref{alg:upper_bound}. We replace the perceptron update step with a kernel-perceptron update step.}
\end{figure}

One challenging aspect of this algorithm is minimizing $f_{T, \alpha}^k(z)$. For linear classifiers, this has a closed form solution that utilizes the dual norm. For arbitrary Kernel classifiers, this is a somewhat more challenging problem. However, we note that this can be solved using standard optimization techniques, and in some cases (when $K$ is particularly simple), it can be solved with basic gradient descent.

Finally, we show that this Algorithm has similar performance to the linear case. Instead of using the robust aspect ratio, $\rho_r(\D)$, to bound the performance, we will require the \textbf{robust $K$-aspect ratio}, which is the kernel analog of this quantity. It can be thought of as the robust aspect ratio in the embedded space $H$. Details about this quantity (along with the proof of the theorem) can be found in Appendix \ref{sec:kernel_appendix}.

\begin{thm}\label{thm:upper_bound_kernel}
Let $\D$ be a distribution with robust $K$-aspect ratio $\rho_r^K(\D)$. Then for any $n > 0$, we have $$\E_{S \sim \D^n} [\L_r(A_S, \D)] \leq O(\frac{\rho_r^K(\D)^2}{n}),$$ where $A_S$ denotes the classifier learned by Algorithm \ref{alg:upper_bound_kernel} from training data $S$. 
\end{thm}

This result indicates that for small values of $\rho_r^k(\D)$, we can achieve a very good robust sample complexity for kernel classifiers. However, as the size of the perturbations approach this margin, this quantity goes to infinity. This phenomenon mirrors the linearly separable case, and suggests that a similar overall dynamic holds for kernel classification. We leave finding a full generalization (including our lower bound) for a direction in future work. 

\section*{Acknowledgments}

We thank NSF under CNS 1804829 for research support.

\bibliography{references}
\bibliographystyle{icml2021}

\onecolumn
\newpage

\appendix

\section{Expanded summary of \cite{ravikumar20}}\label{sec:appendix_comparison}

In this section, we derive the formulation of Theorem \ref{thm:ravikumar} directly from their results. In particular, their results are not stated in terms of $L_{rob}$ and $L_{std}$, and are instead framed in terms of different parameters. To account for this, we first review these alternative parameters, and then show how the statements in Theorem \ref{thm:ravikumar} can be 

Recall, that \cite{ravikumar20} consider the setting in which the data distribution $\D_{\mu, \Sigma}$ can be characterized as a pair of Gaussians in $\R^d$, $\N(\mu, \Sigma)$ and $\N(-\mu, \Sigma)$, that are symmetric about the origin with each of them representing one label class. They consider robustness measured in any normed metric in $\R^d$, including the $\ell_p$ norm for $p \in [1, \infty]$. 

For any such distribution (and robustness radius $r$), they introduce parameters $s_{rob}(\mu, \Sigma)$ and $s_{std}(\mu, \Sigma)$, which they refer to as the robust and standard signal-to-noise ratios respectively, that are defined as follows:

$$s_{std}(\mu, \Sigma) = 2\sqrt{\mu^t\Sigma^{-1}\mu},$$ $$s_{rob}(\mu, \Sigma) = \min_{||z||_p \leq r} 2\sqrt{(\mu - z)^t\Sigma^{-1}(\mu - z)},$$ where $r$ represents the robustness radius and $\ell_p$ is the distance norm under which adversarial perturbations are measured. 

They then show that these parameters fully characterize the sample complexity for robust and standard learning respectively. They express this through the following results:
\begin{enumerate}
	\item Let $\Phi$ denote the cumulative density function of the standard normal distribution, and let $\overline{\Phi}(x) = 1 - \Phi(x)$. Then for any $\D_{\mu, \Sigma}$, 
		\begin{itemize}
			\item the optimally accurate classifier has standard loss $\Pphi(\frac{1}{2}s_{std})$.
			\item the optimally robust classifier has robust loss $\Pphi(\frac{1}{2}s_{rob})$.
		\end{itemize}		 
	\item For any learning algorithm, there exists some mixture of $\D_{\mu, \Sigma}$ such that the expected robust loss is at least $\Omega(e^{(-\frac{1}{8} + o(1))s_{rob}^2}\frac{d}{n})$.
	\item By contrast, for any distribution $\D_{\mu , \Sigma}$, it is possible to learn a classifier with expected standard loss at most $O(s_{std}e^{-\frac{1}{8}s_{std}^2}\frac{d}{n})$.
	\item Thus, by (2.) and (3.), the gap between the robust sample complexity and the standard complexity can be bounded as $$gap \geq \Omega\left(\frac{e^{(-\frac{1}{8} + o(1))s_{rob}^2}\frac{d}{n}}{s_{std}e^{-\frac{1}{8}s_{std}^2}\frac{d}{n}}\right) \simeq \Omega(e^{\frac{-1}{8}(s_{std}^2 - s_{rob}^2)}).$$ They then qualitatively analyze this gap, and observe that for large values of $\mu$ and large values of $r$, this gap can be arbitrarily large, even as a function of $d$, the dimension.
\end{enumerate}

We now show how to convert (2.), (3.), and (4.) into the statements appearing in Theorem \ref{thm:ravikumar}. As before, let us define $L_{std}$ and $L_{rob}$ as the best possible standard and robust losses for $\D_{\mu , \Sigma}$ respectively. In particular, by (1.), we have $$L_{std} = \Pphi(\frac{1}{2}s_{std}^2),\text{ and }L_{rob} = \Pphi(\frac{1}{2}s_{rob}^2).$$ We now express the bounds in (2.) and (4.) in terms of $L_{std}$ and $L_{rob}$. To do so, we use the well known inequality bounding $\Pphi(x)$ as $$\Omega(\frac{x}{x^2 + 1}e^{-x^2/2}) < \Phi(x) <  O(\frac{e^{-x^2/2}}{x}).$$ Substituting this into (2.) through (4.) imply the following, alternative forms.

\begin{enumerate}
	\item[2.] For any learning algorithm, there exists some mixture of Gaussians, $\D_{\mu, \Sigma}$ such that the expected robust loss is at least $\Omega(L_{rob}\frac{d}{n}).$
	\item[3.] For any distribution $\D_{\mu, \Sigma}$, it is possible to learn a classifier with expected standard loss at most $O(L_{std}\frac{d}{n})$.
	\item[4.] By (2.) and (3.), the gap between robust sample complexity and standard sample complexity can be expressed as $$gap \geq \Omega(\frac{L_{rob}}{L_{std}}).$$
\end{enumerate}

Together, these three statements comprise Theorem \ref{thm:ravikumar}. 

\subsection{The limiting case}

While a core difference between our works is that we consider separated distributions whereas Gaussians are non-separated, we now consider the limiting case in which a pair of Gaussians \textit{appear} separated. To do this, we will consider a case in which $L_{rob}$ is small, and $n \sim O(\frac{1}{L_{rob}})$. In this case, with high probability, a sample of size $n$ will \textit{appear} linearly $r$-separated. Examining the bound in part 1 of Theorem \ref{thm:ravikumar}, we see that their lower bound on the expected robust loss reduces to $O(\frac{1}{n}\frac{d}{n}) = O(\frac{d}{n^2})$, which is significantly weaker than ours (Theorem \ref{thm:lower}). Thus, considering Gaussians that appear linearly $r$-separated does not generalize to the general, linearly $r$-separated case.

\section{Proof of Theorem \ref{thm:lower}}

We begin by broadly outlining our proof of Theorem \ref{thm:lower}. Let $\A$ be a probability distribution over $\F_{r, \rho}$, and let $A$ be a learning algorithm that returns a linear classifier.

\begin{enumerate}
	\item Sample $\D \sim \A$.
	\item Sample $S \sim \D^n$.
	\item Learn the classifier $A_S$ using algorithm $A$ and training sample $S$.
	\item Evaluate $A_S$ on $\D$. That is, compute $\L_r(A_S, \D)$. 
\end{enumerate}

The basic idea of our proof is to show that for an appropriate choice of $\A$, the overall expected loss of this procedure, $\L_r(A_S, \D)$, satisfies  $$\E_{D \sim \A}[\E_{S \sim \D^n}[\L_r(A_S, \D)]] \geq \Omega(\frac{d}{n}).$$ Our primary method for doing this is switching expectations. In particular, observe that $$\E_{D \sim \A}[\E_{S \sim \D^n}[\L_r(A_S, \D)]] = \E_{S \sim \B}[\E_{\D \sim \Pi|S}[\L_r(A_S, \D)]],$$ where $\B$ denotes the distribution over all $S$ obtained from first sampling $\D \sim \A$ and then sampling $S \sim \D^n$, and $\Pi|S$ denotes the posterior distribution of $\D$ after observing $S$. It then suffices to bound the quantity $\E_{\D \sim \Pi|S}[\L_r(A_S, \D)]$, which is a significantly more tractable problem since we no longer need to deal with any specifics of the Algorithm $A$. In particular, $S$ is fixed in this expectation and consequently $A_S$ is just a fixed linear classifier. This bound subsequently follows from the distribution $\Pi|S$ having enough ``variation" for this expectation to be sufficient large. 

Our proof will have the following main steps, each of which is given its own subsection.

\begin{enumerate}
	\item In section \ref{subsec:constructing_A}, we construct the distribution $\A$, and prove several important properties about it. 
	\item In section \ref{subsec:bound_expectation}, we show that the desired property of $\A$ holds, by bounding $\E_{\D \sim \A|S}[\L_r(A_S, \D)].$
\end{enumerate}

\subsection{Constructing $\A$}\label{subsec:constructing_A}

We let $r$ be a fixed robustness radius, and $\ell_p$ be our norm with which we measure robustness. Our construction of $\A$ is a somewhat technical and lengthy process. We will organize this construction into 4 subsections, outlined here:
\begin{itemize}
	\item In section \ref{subsubsec:D_a}, we define the distribution $\D_a$, characterized by parameter $a \in [0,1]^d$. This forms the basis for constructing $\A$, which will comprise of distributions $\D_a$ for certain choices of $a$. We also show that $\D_a$ is linearly $r$-separated.
	\item In section \ref{subsubsec:dd}, we define the constant $\dd$, which will be essential for specifying which values of parameter $a$ are permissible. 
	\item In section \ref{subsubsec:g1g2}, we define functions $g_1, g_2: [0, \frac{\dd}{3}] \to [0, \frac{\dd}{3}]$ that will be used to construct $\A$. 
	\item In section \ref{subsubsec:finalA}, we finally put together the previous 3 sections and construct $\A$. We also show that any $\D_a \sim \A$ satisfies $\rho(\D_a) \leq C$. 
\end{itemize}

\subsubsection{Defining $\D_a$}\label{subsubsec:D_a}

Let $e_1, e_2, \dots, e_d$ denote the standard normal basis in $\R^d$. Define $v_i = R e_i$ and $u = \frac{\rr}{\sqrt{d}} \sum_1^d e_i$, where $\rr = \frac{9rd^{1/q}}{2\sqrt{d}}$. It will also be convenient to define the following function, which we will frequently use throughout the entirety of the appendix.
\begin{defn}\label{defn:function_f}
For $1 \leq l \leq \infty$, let $\f_l: [0,1]^d \to \R^+$ be the function defined as $$\f_l(a) = \sqrt[l]{\sum_1^d|\frac{1}{\sqrt{d}} + \overline{a} - a_i|^l},$$ where $\overline{a} = \frac{1}{d}\sum_1^d a_i$. For $l = \infty$, we take the convention that $\sqrt[\infty]{\sum_1^d |x_i|^\infty} = \max_{1 \leq i \leq d} |x_i|.$ 
\end{defn}

To define $\D_a$, we first define the concept of a line segment in $\R^d$.
\begin{defn}\label{defn:line_segment}
Let $x_1, x_2 \in \R^d$ be two points. A \textbf{line segment} joining $x_1, x_2$ is defined as one of the following four sets. 
\begin{itemize}
	\item $(x_1, x_2) = \{tx_1 + (1-t)x_t: 0 < t < 1\}$.
	\item $[x_1, x_2) = \{tx_1 + (1-t)x_t: 0 \leq t < 1\}$.
	\item $(x_1, x_2] = \{tx_1 + (1-t)x_t: 0 < t \leq 1\}$.
	\item $[x_1, x_2] = \{tx_1 + (1-t)x_t: 0 \leq t \leq 1\}$.
\end{itemize}
We will always distinguish which set we mean by using the notation above. In all cases, $x_1, x_2$ are said to be the endpoints of the line segment. 
\end{defn}
We now define $\D_a$.
\begin{defn}\label{def:w_dist}
Let $a \in [0,1]^d$ be a vector, and let $\overline{a} = \frac{1}{d}\sum_1^d a_i$. Set $\lambda_a = \frac{r}{\rr}f_q(a)$, where $q$ is the dual norm of $p$. Assume that for all $1 \leq i \leq d$, $a_i > \lambda_a$ (i.e. we only $\D_a$ for $a$ for which this holds). Let $S^-$ and $S^+$ be two sets of $d$ disjoint line segments (as defined in Definition \ref{defn:line_segment}) defined as $$S^- = \{[v_i, v_i + (a_i - \lambda_a)u): 1 \leq i \leq d\},$$ $$S^+ = \{(v_i + (a_i + \lambda_a)u, v_i + u]: 1 \leq i \leq d\}.$$ Then $D_a$ is defined as the probability distribution of random variables $(X,Y)$ where 
\begin{itemize}
	\item $X$ is chosen by the following random procedure. First, sample an arbitrary segment from $S^+ \cup S^-$ with each segment chosen with probability proportional to its $\ell_2$ length. Next, $X$ is selected from the uniform distribution over the chosen line segment. In particular, the probability that $X$ lies on any interval on any line segment contained within $S^+ \cup S^-$ is directly proportional to the length of the interval. 
	\item $Y$ is $-1$ if $X \in \cup S^-$ and $+1$ if $X \in \cup S^+$.
\end{itemize}
\end{defn}
\begin{figure}[h]
\vspace{.3in}
\includegraphics[scale=0.5]{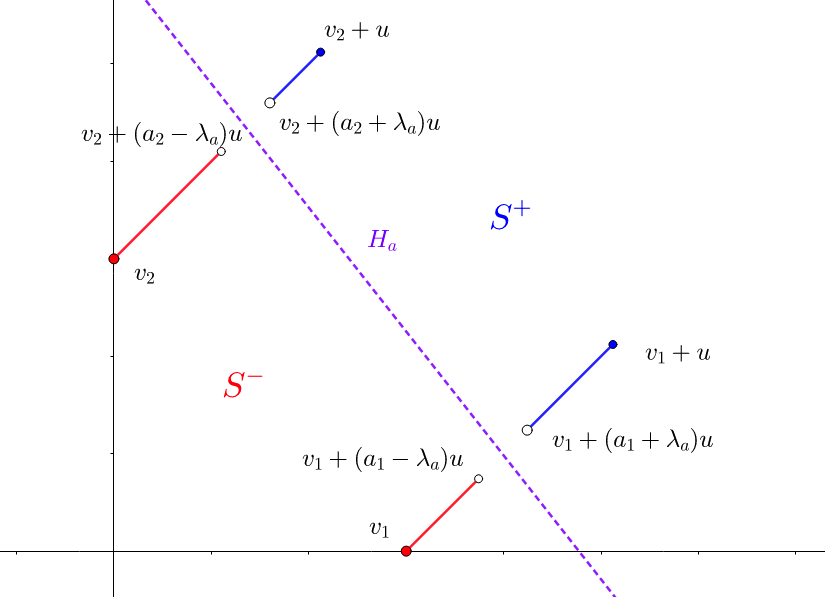}
\vspace{.3in}
\caption{An illustration of $\D_a$ in two dimensions. $S^-$ is shown in red, and $S^+$ is shown in blue. The decision boundary, $H_a$, of the optimal linear classifier, $f_{w^a, 1}$, is shown in purple. }
\label{fig:d_a_illustration}
\end{figure}

We include an example of such a distribution in Figure \ref{fig:d_a_illustration}. Next, we explicitly compute a linear classifier that linearly $r$-separates $\D_a$.

\begin{defn}\label{def:normal_vector}
Let $a \in [0,1]^d$, and let $\overline{a} = \sum_{i=1}^d a_i.$ Then let $w^a$ be defined as $$w_i^a = \frac{1}{\rr} - \frac{da_i}{\rr\sqrt{d} + d\rr\overline{a}}.$$ 
\end{defn}

\begin{lem}\label{lem:normal_vector_works}
$w^a$ satisfies $\langle w^a, u\rangle = \frac{d}{\sqrt{d} + d\overline{a}}$ and $\langle w^a, v_i + a_iu \rangle = 1,$ for all $1 \leq i \leq d.$ 
\end{lem}

\begin{proof}
By the definitions of $v_i, u$, we have that 
\begin{equation*}
\begin{split}
\langle w^a, u \rangle &= \langle w^a, \frac{1}{\sqrt{d}}\sum_1^d v_i \rangle \\
&= \frac{1}{\sqrt{d}} \sum_1^d Rw_i^a \\
&= \frac{1}{\sqrt{d}} \sum_1^d 1 - \frac{da_i}{\sqrt{d} + d\overline{a}} \\
&= \frac{1}{\sqrt{d}} \sum_1^d \frac{\sqrt{d} + d\overline{a} - da_i}{\sqrt{d} + d\overline{a}} \\
&= \frac{1}{\sqrt{d}} \frac{d\sqrt{d}}{\sqrt{d} + d\overline{a}} = \frac{d}{\sqrt{d} + d\overline{a}},
\end{split}
\end{equation*}
Which proves the first claim. Next, we also have that $\langle w^a, v_i \rangle = Rw_i^a$. Summing these, we get $$Rw_i^a + \frac{da_i}{\sqrt{d} + d\overline{a}} = 1 - \frac{da_i}{\sqrt{d} + d\overline{a}} + \frac{da_i}{\sqrt{d} + d\overline{a}} = 1,$$ as desired. 
\end{proof}

We now prove that $\D_a$ is linearly $r$-separated.

\begin{lem}\label{lem:separation}
$\D_a$ is linearly $r$-separated by the classifier $f_{w_a, 1}$.
\end{lem}

\begin{proof}
Let $H_a$ denote the hyperplane passing through $\{v_i + a_iu: 1 \leq i \leq d\}$. By Lemma \ref{lem:normal_vector_works}, $H_a$ is the decision boundary of $f_{w_a, 1}$.  Referring to Figure \ref{fig:d_a_illustration}, we see that $\cup S^+$ lies entirely above $H_a$ while the set $\cup S^-$ lies entirely below the hyperplane $H_a$, which the classifier $f_{w^a, 1}$ has accuracy $1$ with respect to $\D_a$. It suffices to show that $f_{w^a, 1}$ is robust everywhere. In order to do this, we must show that all points in the support of $\D_a$ have $\ell_p$ distance at least $r$ from $H_a$. 

Fix any $1 \leq i \leq d$. Since the $\ell_p$ distance metric is invariant under translation and scales linearly with dilations, it follows that the point $x_i = v_i + (a_i - \lambda_a)u$ is the closest point on the segment $[v_i, v_i+(a_i - \lambda_a)u)$ to $H_a$. Suppose $x_i$ has distance $D$ under the $\ell_p$ norm to $H_a$. Then the key observation is that the $\ell_p$ ball, $B_p(x_i, D)$, must be tangent to $H_a$. Expressing this as an equation, we have $\max_{z \in B_p(x_i, D)} \langle z, w^a \rangle = 1,$ which can be re-written as $$\max_{||z - x_i||_p \leq D} \langle z - x_i, w^a \rangle = 1 - \langle x_i, w^a \rangle.$$ By Lemma \ref{lem:normal_vector_works} , $\langle w^a, u \rangle = \frac{d}{\sqrt{d} + d\overline{a}}$ and $\langle w^a, v_i + a_iu \rangle = 1$. Substituting this, we see that 
\begin{equation*}
\begin{split}
1 - \langle x_i, w^a \rangle &= 1 - \langle v_i + a_iu - \lambda_a u, w^a \rangle \\
&= 1 - \langle v_i + a_iu, w^a \rangle + \langle \lambda_a u, w^a \rangle \\
&= \langle \lambda_a u, w^a \rangle \\
&= \frac{d\lambda_a}{\sqrt{d} + d\overline{a}}.
\end{split}
\end{equation*}

However, by using the dual norm, we see that $\max_{||z - x_i||_p \leq D} \langle z - x_i, w^a \rangle = D||w^a||_q$. Thus it follows that
\begin{equation*}
\begin{split}
D &= \frac{d\lambda_a}{(\sqrt{d} + d\overline{a})||w^a||_q} \\
&= \frac{d\frac{r}{\rr}f_q(a)}{(\sqrt{d} + d\overline{a})||w^a||_q} \\
&= \frac{d\frac{r}{\rr}\sqrt[q]{\sum_1^d |\frac{1}{\sqrt{d}} + \overline{a} - a_i|^q}}{(\sqrt{d} + d\overline{a})||w^a||_q} \\
&= \frac{r\sqrt[q]{\sum_1^d |\frac{1}{\rr}\frac{\sqrt{d} + d\overline{a} - da_i}{(\sqrt{d} + d\overline{a})}|^q}}{||w^a||_q} \\
&= \frac{r||w^a||_q}{||w^a||_q} = r.
\end{split}
\end{equation*}
We can use an analogous argument holds for $v_i + (a_i + r_a)u$, the closest point to $H_a$ in $S^+$. Thus each point in the support of $D^a$ has distance strictly larger than $r$ (as the endpoints were not included) to $H_a$. Consequently $f_{w^a, 1}$ linearly $r$-separates $D^a$, as desired. 
\end{proof}

\subsubsection{Defining $\dd$}\label{subsubsec:dd}

Now that we have defined $\D_a$, we turn our attention to defining $\A$, which requires us to specify a distribution over valid choices of $a$. In particular, although $\D_a$ is defined for $a \in [0, 1]^d$, we will require a more stringent condition on $a$ for our construction to work. To this end, we begin by defining $\Delta$, a key parameter that characterizes the domain of $a$. To define $\dd$, we use the following lemma.

\begin{lem}\label{lem:dd}
There exists a real number $\dd > 0$ such that for all $l \in \{2, q\}$, and for all $a \in [\frac{1}{2} - \dd, \frac{1}{2} + \dd]^d$, $$||\nabla f_l(a)||_2 \leq \frac{1}{d^2\sqrt{d}},$$ where $f_l$ is as defined in Definition \ref{defn:function_f}.
\end{lem}

\begin{proof}
Since $1 \leq q < \infty$, we see that for both choices of $l$, the function $h_l(x) = (\frac{1}{\sqrt{d}} - x)^l$ is a convex function for $x \in [-\frac{1}{2\sqrt{d}}, \frac{1}{2\sqrt{d}}]$. Thus, if $\sum_1^d x_i = 0$, then by Jensen's inequality, $\sum_1^d h_l(x_i) \geq \sum_1^d h_l(0)$. Applying this, we see that for all $l \in \{2, q\}$ and for all $a \in [\frac{1}{2} - \frac{1}{4\sqrt{d}}, \frac{1}{2} + \frac{1}{4\sqrt{d}}]^d$, 
\begin{equation*}
\begin{split}
f_l(a) &= \sqrt[l]{\sum_1^d |\frac{1}{\sqrt{d}} + \overline{a} - a_i|^l} \\
&= \sqrt[l]{\sum_1^d \left(\frac{1}{\sqrt{d}} + \overline{a} - a_i\right)^l} \\
&= \sqrt[l]{\sum_1^d h_l(a_i - \overline{a})} \\
&\geq \sqrt[l]{\sum_1^d h_l(0)} \\
&= f_l((\frac{1}{2}, \frac{1}{2}, \dots, \frac{1}{2})),
\end{split}
\end{equation*}
with the first equality holding since $\overline{a} - a_i < \frac{1}{\sqrt{d}}$ and the first inequality holding since $\sum_1^d a_i - \overline{a} = 0$. Thus $f_l(a)$ must be locally minimized when $a = (\frac{1}{2}, \frac{1}{2}, \dots, \frac{1}{2})$, and it follows that $$||\nabla f_l(\frac{1}{2}, \frac{1}{2}, \dots, \frac{1}{2})||_2 = 0, \text{ for } l = 2, q.$$ Now observe that the map $H(a) = \max_{l \in \{2, q\}} ||\nabla f_l(a)||_2$ is a continuous map as long as $|a_i - \overline{a}| < \frac{1}{\sqrt{d}}$ for all $1 \leq i \leq d$. Thus there exists an open neighborhood $U$ about $(\frac{1}{2}, \frac{1}{2}, \dots, \frac{1}{2})$ such that $H(a) \leq \frac{1}{d^2\sqrt{d}}$ for all $a \in U$. Taking $\dd$ so that $[\frac{1}{2} - \dd, \frac{1}{2} + \dd]^d \subseteq U$ suffices. 
\end{proof}

\begin{defn}\label{defn:Delta}
Let $\Delta$ be any constant for which Lemma \ref{lem:dd} holds. In particular, $\Delta$ only depends on $\ell_p$, the robustness norm, and $d$, the dimension.
\end{defn}

\subsubsection{Defining $\g_1$ and $\g_2$}\label{subsubsec:g1g2}

In this section, we define functions $\g_1, \g_2: [0, \frac{\Delta}{3}] \to [0, \frac{\Delta}{3}]$ which we will use to specify $\A$. Before defining $\g_1$ and $\g_2$, we will first prove several technical lemmas.

\begin{lem}\label{lem:phi_s}
Let $I \subseteq \R$ be an interval, and $\Phi:I \to \R$ be a strictly convex function.  For any $s \in \R$ and $t \geq 0$, let $\Phi_s(t) = \Phi(s-t) + \Phi(s + t)$. Then $\Phi_s$ is a strictly increasing function.
\end{lem}

\begin{proof}
Fix $s$, and let $0 \leq t_1 < t_2$. Then we see that by Jensen's inequality (for strictly convex functions), $$\Phi(s + t_1) < \frac{(t_2-t_1)\Phi(s+t_2)}{t_1 + t_2} + \frac{2t_1\Phi(s-t_1)}{t_1 + t_2},$$ and $$\Phi(s - t_1) < \frac{(t_2-t_1)\Phi(s-t_2)}{t_1 + t_2} + \frac{2t_1\Phi(s+t_1)}{t_1 + t_2}.$$ Summing these inequalities, we see that 
\begin{equation*}
\begin{split}
\Phi_s(t_1) &= \Phi(s - t_1) + \Phi(s + t_1) \\
&< \frac{(t_2-t_1)\Phi(s+t_2)}{t_1 + t_2} + \frac{2t_1\Phi(s-t_1)}{t_1 + t_2} + \frac{(t_2-t_1)\Phi(s-t_2)}{t_1 + t_2} + \frac{2t_1\Phi(s+t_1)}{t_1 + t_2} \\
&= \frac{t_2 - t_1}{t_1 + t_2}(\Phi(s +t_2) + \Phi(s - t_2)) + \frac{2t_1}{t_1 + t_2}(\Phi(s - t_1) + \Phi(s + t_1)) \\
&= \frac{t_2 - t_1}{t_1 + t_2}\Phi_s(t_2) + \frac{2t_1}{t_1 + t_2}\Phi_s(t_1).
\end{split}
\end{equation*}
Rearranging this yields $\Phi_s(t_1) < \Phi_s(t_2)$, as desired.
\end{proof}

\begin{lem}\label{lem:lipschitz}
Let $I \subseteq \R$ be an interval, $\Phi:I \to \R$ be a strictly convex continuous function, and  $x, y, z \in I$ be real numbers with $x < y < z$. Let $\epsilon >0$ be such that $x - \epsilon \in I$ and $y + \epsilon \leq z - \epsilon$. Then there exist unique $\delta, \gamma >0$ such that the following hold: $$\delta + \gamma = \epsilon,$$ $$\Phi(x-\delta) + \Phi(y + \epsilon) + \Phi(z - \gamma) = \Phi(x) + \Phi(y) + \Phi(z)$$
\end{lem}

\begin{proof}
Fix any $\epsilon$ satisfying the desired conditions, and define $\Theta: [0, \epsilon] \to \R$ as $\Theta(t) = \Phi(x - t) + \Phi(y + \epsilon) + \Phi(z + t - \epsilon)$. Then, utilizing the definition of $\Phi_s$ from Lemma \ref{lem:phi_s}, we see that $$\Theta(t) = \Phi_{\frac{x + z - \epsilon}{2}}(\frac{z - x - \epsilon}{2} + t) + \Phi(y + \epsilon).$$ By Lemma \ref{lem:phi_s}, it follows that $\Theta$ is strictly increasing in $t$, and since $\Phi$ is continuous, so is $\Theta$. Next, we bound $\Theta(0)$ and $\Theta(\epsilon)$ to put us in the configuration to apply the intermediate value theorem. To bound $\Theta(0)$, we have
\begin{equation*}
\begin{split}
\Theta(0) &= \Phi(x) + \Phi(y + \epsilon) + \Phi(z - \epsilon) \\
&= \Phi(x) + \Phi_{\frac{y+z}{2}}(\frac{z - y}{2} - \epsilon) \\
&< \Phi(x) + \Phi_{\frac{y+z}{2}}(\frac{z - y}{2}) \\
&= \Phi(x) + \Phi(y) + \Phi(z),
\end{split}
\end{equation*}
and to bound $\Theta(\epsilon)$, we have 
\begin{equation*}
\begin{split}
\Theta(\epsilon) &= \Phi(x - \epsilon) + \Phi(y + \epsilon) + \Phi(z) \\
&= \Phi_{\frac{x+y}{2}}(\frac{y- x}{2} + \epsilon) + \Phi(z)\\
&> \Phi_{\frac{x+y}{2}}(\frac{y - x}{2}) + \Phi(z) \\
&= \Phi(x) + \Phi(y) + \Phi(z).
\end{split}
\end{equation*}
Together, these equations imply $\Theta(0) < \Phi(x) + \Phi(y) + \Phi(z) < \Theta(\epsilon)$. Since $\Theta$ is strictly increasing and continuous, there exists a unique $\delta \in [0, \epsilon]$ such that $\Theta(\delta) = \Phi(x) + \Phi(y) + \Phi(z)$. Setting $\gamma = \epsilon - \delta$, we see that $$\Theta(\delta) = \Phi(x - \delta) + \Phi(y + \epsilon) + \Phi(z - \gamma) = \Phi(x) + \Phi(y) + \Phi(z),$$ as desired.

\end{proof}

Next, we define a function that will be useful for simplifying notation, both in this section and subsequent ones.

\begin{defn}\label{defn:F_the_function}
Let $\Delta$ be as in definition \ref{defn:Delta}. For $x, y, z \in [0, \frac{\dd}{3}]$, let $$F(x, y, z) = \sqrt[q]{\left(\frac{1}{\sqrt{d}} - x\right)^q + \left(\frac{1}{\sqrt{d}} - \frac{2\Delta}{3} + y\right)^q + \left(\frac{1}{\sqrt{d}} + \frac{2\Delta}{3} + z\right)^q}.$$
\end{defn}

We now define $\g_1, \g_2$.
\begin{cor}\label{cor:lipschitz_maps}
Let $\Delta$ be as in definition \ref{defn:Delta}. There exist $1$-Lipshitz, monotonically non-decreasing functions $\g_1, \g_2: [0, \frac{\dd}{3}] \to [0, \frac{\dd}{3}]$ such that for all $t \in [0, \frac{\dd}{3}]$, $\g_1(t) + \g_2(t) = t$ and $F(t, \g_1(t), \g_2(t)) = F(0, 0, 0)$. 
\end{cor}
\begin{proof}
We have two cases.
\paragraph{Case 1: $1 < q < \infty$:} Let $\Phi: [-\Delta, \Delta] \to \R$ be defined as $\Phi(x) = (\frac{1}{\sqrt{d}} - x)^q$. Since $q > 1$, and $\Delta < \frac{1}{\sqrt{d}}$, $\Phi$ is strictly convex. Observe that $$F(x, y, z)^q = \Phi(x) + \Phi(2\frac{\dd}{3} - y) + \Phi(-2\frac{\dd}{3} - z).$$ Next, fix any $t \in [0, \frac{\Delta}{3}]$. Then observe that $-\frac{2\Delta}{3} \geq -\Delta$ and that $\frac{2\Delta}{3} - t \geq 0 + t$. This puts us in the configuration to apply Lemma \ref{lem:lipschitz}. In particular, there exist unique reals $\delta_t, \gamma_t > 0$ such that $$\delta_t + \gamma_t = t,$$ $$\Phi(-\frac{2\Delta}{3} - \delta_t) + \Phi(t) + \Phi(\frac{2\Delta}{3} - \gamma_t) = \Phi(-\frac{2\Delta}{3}) + \Phi(0) + \Phi(\frac{2\Delta}{3}).$$ We now define $g_1, g_2: [0, \frac{\Delta}{3}] \to [0, \frac{\Delta}{3}]$ as $$g_1(t) = \gamma_t\text{ and }g_2(t) = \delta_t.$$ Then it is clear that $F(0,0,0) = F(t, g_1(t), g_2(t))$ and $g_1(t) + g_2(t)$ (by directly substituting into the equations above). All that remains is to show that $g_1$ and $g_2$ are 1-Lipschitz. 

Fix any $0 \leq t_1 < t_2 \leq \frac{\Delta}{3}$, and let $t_2 - t_1 = \epsilon$. The key idea is to apply Lemma \ref{lem:lipschitz} to $-\frac{2\Delta}{3} - g_2(t_1)<  t_1 < \frac{2\Delta}{3} - g_1(t_1)$ and $\epsilon$. To do so, we first check the conditions of the lemma. 

We have that $$-\frac{2\Delta}{3} - g_2(t_1) - \epsilon \geq -\frac{2\Delta}{3} - t_1 - \epsilon = -\frac{2\Delta}{3} - t_2 \geq -\Delta,$$ and 
\begin{equation*}
\begin{split}
t_1 + \epsilon &= t_2 \\
&\leq \frac{\Delta}{3} \\
&\leq \frac{2\Delta}{3} - t_2 \\
&= \frac{2\Delta}{3} - t_1 - \epsilon \\
&\leq \frac{2\Delta}{3} - g_1(t_1) - \epsilon.
\end{split}
\end{equation*}

Thus $\epsilon$ satisfies the necessary conditions for Lemma \ref{lem:lipschitz}. Since $\Phi$ is strictly convex, by Lemma \ref{lem:lipschitz}, there exist unique $\delta, \gamma > 0$ with $\delta+ \gamma = \epsilon$ such that $$\Phi(-\frac{2\Delta}{3} - g_2(t_1) - \delta) + \Phi(t_1 + \epsilon) + \Phi(\frac{2\Delta}{3} - g_1(t_1) - \gamma) = \Phi(-\frac{2\Delta}{3} - g_2(t_1)) + \Phi(t_1) + \Phi(\frac{2\Delta}{3} - g_1(t_1)).$$ However, by the definition of $g_1, g_2$, we see that both of these quantities are equal to $F(0,0,0)^q$. Moreover, again by the definition of $g_1, g_2$, we also have that $g_1(t_2)$ and $g_2(t_2)$ are the unique real numbers in $[0, \frac{\Delta}{3}$ that satisfy $$\Phi(-\frac{2\Delta}{3} - g_2(t_2)) + \Phi(t_2) + \Phi(\frac{2\Delta}{3}+g_1(t_2)) = F(0,0,0)^q.$$ Thus, it follows that $g_2(t_2) = g_2(t_1) + \delta$ and $g_1(t_2) = g_1(t_1) + \gamma$. However, $t_2 - t_1 = \epsilon$, and $\delta, \gamma < \epsilon$ (since they sum to $\epsilon$). Thus, we see that $|g_1(t_2) - g_1(t_1)| \leq |t_2 - t_1|$ and $|g_2(t_2) - g_2(t_1)| \leq |t_2 - t_1|$. Since $t_1$ and $t_2$ were arbitrary, it follows that $g_1$ and $g_2$ are both $1$-Lipschitz, as desired. 

Finally, since $\delta, \gamma > 0$, it follows that $g_2(t_2) > g_2(t_1)$ and $g_1(t_2) > g_1(t_1)$. Since $t_1, t_2$ were arbitrary, it follows that $g_1, g_2$ are monotonically non-decreasing.

\paragraph{Case 2: $q = 1$} In this case, since $\Delta < \frac{1}{\sqrt{d}}$ (Lemma \ref{lem:dd}), we see that $F(x, y, z) = \frac{3}{\sqrt{d}} + y + z - x$. Setting $\g_1(t) = \g_2(t) = \frac{t}{2}$ suffices, and clearly satisfies the desired properties. 
\end{proof}

\begin{defn}\label{defn:g_1_and_g_2}
Let $\Delta$ be as defined in Definition \ref{defn:Delta}. We let $g_1, g_2: [0, \frac{\Delta}{3}] \to [0, \frac{\Delta}{3}]$ be defined as any function satisfying the conditions of Corollary \ref{cor:lipschitz_maps}.
\end{defn}

\subsubsection{Putting it all together: defining $\A$}\label{subsubsec:finalA}

We are now ready to define $\A$. For convenience, we assume $d$ is a multiple of $3$.
\begin{defn}\label{defn:A}
Let $\Delta, g_1$, and $g_2$ be as defined in Definitions \ref{defn:Delta} and \ref{defn:g_1_and_g_2}. Then $\A$ is defined as the distribution of distributions $\D_a$ where $a$ is a random vector constructed as follows. Let $t_1, t_2, \dots t_{d/3}$ be drawn i.i.d from the uniform distribution over $[0, \frac{\dd}{3}]$. Then for $1 \leq i \leq d/3$, we let
	\begin{itemize}
		\item $a_i = \frac{1}{2} + t_i$.
		\item $a_{i+d/3} = \frac{1}{2} + 2\frac{\dd}{3} - g_1(t_i)$.
		\item $a_{i + 2d/3} = \frac{1}{2} - 2\frac{\dd}{3} - g_2(t_i).$
	\end{itemize}
Together the variables $a_1, a_2, \dots, a_d$ compose $a$. Thus a random distribution $\D \sim \A$ can be constructed by sampling $a$ as above and setting $\D = \D_a$.
\end{defn}

We now show that for all $\D_a \sim \A$, $\lambda_a$ (Definition \ref{def:w_dist}) is constant.

\begin{lem}\label{lem:cons_lambda}
There exists a constant $\Lambda$ such that for all $\D_a \sim \A$, $\lambda_a = \Lambda$. 
\end{lem}

\begin{proof}
Let $\D_a \sim \A$ be arbitrary. By Lemma \ref{cor:lipschitz_maps}, for all $1 \leq i \leq d$, $g_1(t_i) + g_2(t_i) = t_i$. Substituting this, we see that 
\begin{equation*}
\begin{split}
\overline{a} &= \frac{1}{d}\sum_1^d a_i \\
&= \frac{1}{d}\sum_1^{d/3} (\frac{1}{2} + t_i) + (\frac{1}{2} + \frac{2\Delta}{3} - g_1(t_i)) + (\frac{1}{2} - \frac{2\Delta}{3} - g_2(t_i)) \\
&= \frac{1}{d}\sum_1^{d/3} \frac{3}{2} \\
&= \frac{1}{2}.
\end{split}
\end{equation*}

Recall that $\lambda_a = \frac{r}{R}f_q(a) = \frac{r}{R}\sqrt[q]{\sum_1^d |\frac{1}{\sqrt{d}} + \overline{a} - a_i|^q}$. By substituting that $\overline{a} = \frac{1}{2}$ and expressing each $a_i$ in terms of $t_i$, we see that 
\begin{equation*}
\begin{split}
\lambda_a &=  \frac{r}{R}\sqrt[q]{\sum_1^d |\frac{1}{\sqrt{d}} + \overline{a} - a_i|^q}\\
&= \frac{r}{R}\sqrt[q]{\sum_{i=1}^{d/3} \left|\frac{1}{\sqrt{d}} + \frac{1}{2} - (\frac{1}{2} + t_i)\right|^q + \left|\frac{1}{\sqrt{d}} + \frac{1}{2} - \left(\frac{1}{2} + \frac{2\Delta}{3} - g_1(t_i)\right)\right|^q +  \left|\frac{1}{\sqrt{d}} + \frac{1}{2} - \left(\frac{1}{2} - \frac{2\Delta}{3} - g_2(t_i)\right)\right|^q} \\
 &= \frac{r}{R}\sqrt[q]{\sum_1^{d/3} \left|\frac{1}{\sqrt{d}} - t_i\right|^q + \left|\frac{1}{\sqrt{d}} + g_1(t_i) - \frac{2\Delta}{3}\right|^q + \left|\frac{1}{\sqrt{d}} + g_2(t_i) + \frac{2\Delta}{3}\right|^q} \\
&= \frac{r}{R}\sqrt[q]{\sum_1^{d/3}F(t_i, g_1(t_i), g_2(t_i))^q}, \\
\end{split}
\end{equation*}
where $F$ is defined as in Definition \ref{defn:F_the_function}. Next, by Corollary \ref{cor:lipschitz_maps}, $F(t_i, g_1(t_i), g_2(t_i)) = F(0,0,0)$ for all $1 \leq i \leq \frac{d}{3}$. Thus, if we set $\Lambda = \frac{r}{R}(\frac{d}{3})^{1/q}F(0,0,0)$, we have 
\begin{equation*}
\begin{split}
\lambda_a &= \frac{r}{R}\sqrt[q]{\sum_1^{d/3} F(t_i, g_1(t_i), g_2(t_i))^q} \\
&= \frac{r}{R}\sqrt[q]{\sum_1^{d/3}F(0,0,0)^q} \\
&= \frac{r}{R}\sqrt[q]{\frac{d}{3}F(0,0,0)^q} \\
&=  \frac{r}{R}(\frac{d}{3})^{1/q}F(0,0,0) = \Lambda,
\end{split}
\end{equation*}
proving the claim.
\end{proof}

\begin{defn}\label{defn:big_lambda}
We define $\Lambda = \frac{r}{R}(\frac{d}{3})^{1/q}F(0,0,0)$, where $F$ is defined as in Definition \ref{defn:F_the_function}.
\end{defn}

Next, we compute upper and lower bounds on $\Lambda$, both of which will be useful for subsequent lemmas. 
\begin{lem}\label{lem:lambda_bounds}
$\frac{1}{9} < \Lambda < \frac{1}{3}$. 
\end{lem}

\begin{proof}
By definition, $\Lambda = \frac{d}{3}^{1/q}F(0, 0, 0)$. Substituting the definition of $f$, we see that  $F(0, 0, 0) = \sqrt[q]{|\frac{1}{\sqrt{d}}|^q + |\frac{1}{\sqrt{d}} - \frac{2\Delta}{3}|^q + |\frac{1}{\sqrt{d}} + \frac{2\Delta}{3}|^q},$ and consequently, $$3^{1/q}|\frac{1}{\sqrt{d}} - \frac{2\Delta}{3}| \leq F(0, 0, 0) \leq 3^{1/q}|\frac{1}{\sqrt{d}} + \frac{2\Delta}{3}|.$$ By definition, $\frac{2\Delta}{3} < \frac{1}{2\sqrt{d}}$. It follows that $$\frac{r}{R}\frac{d^{1/q}}{2\sqrt{d}} < \Lambda < \frac{r}{R}\frac{3d^{1/q}}{2\sqrt{d}}.$$ Finally, since $\frac{r}{R} = \frac{2\sqrt{d}}{9d^{1/q}}$, substituting this yields $\frac{1}{9} < \Lambda < \frac{1}{3}$, as desired.
\end{proof}

Next, we show that for all $\D_a \in \A$, the aspect ratio (Definition \ref{defn:aspect_ratio}), $\rho(\D_a)$, is bounded by a constant.

\begin{lem}\label{lem:large_margin}
For all $\D_a \in \A$, we have $\rho(\D_a) \leq 18\sqrt{3}$. 
\end{lem}

\begin{proof}
We first bound the $\ell_2$ margin, $\gamma(\D_a)$ (Definition \ref{defn:margin}). Recall that the margin, $\gamma(\D_a)$ is described as the largest possible $\ell_2$ distance from the support of $\D_a$ to the decision boundary of a linear classifier. Thus, we can lower bound $\gamma(\D_a)$ by computing the distance from the support of $\D_a$ to $H_a$, the decision boundary of $f_{w^a, 1}$ (Definition \ref{def:normal_vector}).

By referring to Figure \ref{fig:d_a_illustration} (in Section \ref{subsubsec:D_a}), it becomes clear that the closest point (under the $\ell_2$ margin) from $S^-$ to $H_a$ is the point $v_i + (a_i - \lambda_a)u$, for some value of $i$. Thus it suffices to compute the $\ell_2$ distance from this point to the plane $H_a$. 

Recall that by Lemma \ref{lem:normal_vector_works}, the point $v_i + a_iu$ satisfies $\langle w^a, v_i + a_iu \rangle = 1$, and consequently must lie on the hyperplane $H_a$. Let $D$ denote the $\ell_2$ distance from $v_i + (a_i - \lambda_a)u$ to $H_a$. Since $w^a$ is the normal vector to $H_a$, it follows that
\begin{equation*}
\begin{split}
D &= \langle v_i + a_iu - (v_i + (a_i - \lambda_a)u), \frac{w^a}{||w^a||_2} \rangle \\
&= \frac{\langle \lambda_a u, w^a \rangle}{||w^a||_2} \\
&\numeq{1} \frac{\langle \Lambda u, w^a \rangle}{||w^a||_2} \\
&\numeq{2} \frac{\Lambda \frac{d}{\sqrt{d} + d\overline{a}}}{||w^a||_2}  \\
&\numeq{3} \frac{\Lambda \frac{d}{\sqrt{d} + d\overline{a}}}{\sqrt{\sum_1^d \left(\frac{\sqrt{d} + d\overline{a} - da_i}{R(\sqrt{d} + d\overline{a}}\right)^2}} \\
&= \frac{R\Lambda}{\sqrt{\sum_1^d (\frac{1}{\sqrt{d}} + \overline{a} - a_i)^2}} \\
&\numeq{4} \frac{R\Lambda}{f_2(a)}.
\end{split}
\end{equation*}

Here, (1) holds by Lemma \ref{lem:cons_lambda}, (2) holds by Lemma \ref{lem:normal_vector_works}, (3) holds by Definition \ref{def:normal_vector}, and (4) holds by Definition \ref{defn:function_f}.

Next, observe that since $\D_a \sim \A$, we must have $a \in [\frac{1}{2} - \Delta, \frac{1}{2} + \Delta]^d$. Thus it follows that $||a - (\frac{1}{2}, \frac{1}{2}, \dots, \frac{1}{2})||_2 \leq \Delta\sqrt{d}$. However, by applying Lemma \ref{lem:dd}, we also see that $f_2$ is $\frac{1}{d^2\sqrt{d}}$-Lipschitz over $[\frac{1}{2} - \Delta, \frac{1}{2} + \Delta]^d$. Thus, it follows that $$f_2(a) \leq f_2(\frac{1}{2}, \frac{1}{2}, \dots, \frac{1}{2}) + \Delta\sqrt{d} \frac{1}{d^2\sqrt{d}} \leq 2,$$ with the latter inequality holding from the definition of $\Delta$. 

Substituting this and applying Lemma \ref{lem:lambda_bounds}, we see that $$\gamma(\D_a) \geq \frac{R\Lambda}{2} \geq \frac{R}{18}.$$ Next, to bound the aspect ratio, $\rho(\D_a)$, we must also bound the $\ell_2$ diameter of $\D_a$. However, the $\ell_s$ diameter of $\D_a$ is $R\sqrt{3}$, since it is the distance from $v_i + u$ to $v_j$ for $i \neq j$. Thus, it follows that $$\rho(\D_a) = \frac{diam_2(\D_a)}{\gamma(\D_a)} \leq \frac{R\sqrt{3}}{R/18} = 18\sqrt{3},$$ as desired. 
\end{proof}

Note that a tighter analysis (and selection of $\Delta$) can give a smaller bound for $\rho(\D_a)$, but the most important fact is that $\rho(\D_a) = O(1)$. 

\subsection{Bounding the expected robust loss}\label{subsec:bound_expectation}

In this section, we finally prove our lower bound, Theorem \ref{thm:lower}. This will require a few important steps, which we have separated into the following subsections. 
\begin{itemize}
	\item In section \ref{subsubsec:loss_bounding}, we give a useful lower bound for the loss $\L_r(f, \D_a)$ where $f$ is an arbitrary linear classifier. 
	\item In section \ref{subsubsec:posterior}, we give an explicit computation for the posterior distribution $\A|S$ where $S \sim \D_a^n$ is the observed training sample. 
	\item Finally, in section \ref{subsubsec:proof}, we present the proof of Theorem \ref{thm:lower}.
\end{itemize}

\subsubsection{Bounding the loss $\L_r(f, \D_a)$}\label{subsubsec:loss_bounding}

In this section, we find a lower bound on the loss $\L_r(f, \D_a)$ where $f$ is a linear classifier. We begin by first restricting $f$ to be in the set of classifiers $$f \in \{f_{w^b, 1}: b \in [0, 1]^d\},$$ where $w^b$ is as defined in Definition \ref{def:normal_vector}. These are precisely the classifiers that have a decision boundary that passes through some point on every line segment in $\{[v_i, v_i + u]: 1 \leq i \leq d\}$. We are able to only consider these classifiers since all other linear classifiers clearly have a very high loss with respect to $\D_a$ as they necessarily misclassify at least half the points on the line segment $[v_i, v_i + u]$ for some value of $i$. 

We now find an initial lower bound on $\L_r(f_{w^b, 1}, \D_a)$.

\begin{lem}\label{lem:loss_bound_general}
Fix any $\D_a \in \A$, and let $b \in [0,1]^d$ be arbitrary. Let $w^b$ be the vector defined as in Definition \ref{def:normal_vector}, and $\lambda_b = \frac{r}{\rr}\f_q(b)$ where $f$ is as defined in Definition \ref{defn:function_f}. Then $$\L_r(f_{w^b, 1}, \D_a) \geq \frac{d(\lambda_b - \lambda_a) + \sum_1^d|a_i - b_i|}{d - 2d\Lambda}.$$ 
\end{lem}

\begin{proof}
By Lemma \ref{lem:separation}, $f_{w^b, 1}$ precisely $r$-separates $\D_b$. This implies that for all $1 \leq i \leq d$,
$$f_{w^b, 1}(x) = \begin{cases} 1 & x \in (v_i + (b_i + \lambda_b)u, v_i + u] \\-1 & x \in [v_i, v_i + (b_i - \lambda_b)u) \\ \text{not robust} & x \in [v_i + (b_i - \lambda_b)u, v_i + (b_i + \lambda_b)u] \end{cases}.$$ Without loss of generality, suppose that $b_i \geq a_i$. The key observation is that for all $1 \leq i \leq d$, if $x \in [v_i + (a_i + \lambda_a)u, v_i + (b_i + \lambda_b)u]$, then $f_{w^b, 1}(x) = -1$ for $f_{w^b, 1}$ is not robust at $x$. In both cases, we see that $f_{w^b, 1}$ is either inaccurate or not robust for all points in $[v_i + (a_i + \lambda_a)u, v_i + (b_i + \lambda_b)u]$. 

This interval has $\ell_2$ length at least $(|a_i - b_i| + (\lambda_b - \lambda_a))||u||_2$. Note that in the case that $a_i \leq b_i$ we can get an identical expression. Thus,  combining this for all $i$, we see that $f_{w^b, 1}$ is either inaccurate or not robust for a total length of $[d(\lambda_b - \lambda_a) + \sum_1^d |a_i - b_i|]||u||_2$. Dividing by the total length of the support of $\D_a$, we find that
\begin{equation*}
\begin{split}
\L_r(f_{w^b, 1}, \D_a) &\geq \frac{[d(\lambda_b - \lambda_a) + \sum_1^d |a_i - b_i|]||u||_2}{\sum_1^d ||[v_i, v_i + (a_i - \lambda_a)u) + (v_i + (a_i + \lambda_a)u, v_i + u]||_2} \\
&= \frac{[d(\lambda_b - \lambda_a) + \sum_1^d |a_i - b_i|]||u||_2}{\sum_1^d ||u_2||(1 - 2\lambda_a)} \\
&= \frac{d(\lambda_b - \lambda_a) + \sum_1^d |a_i - b_i|}{d(1 - 2\lambda_a)} \\
&= \frac{d(\lambda_b - \lambda_a) + \sum_1^d |a_i - b_i|}{d - 2d\Lambda},
\end{split}
\end{equation*}
with the last equality holding since by Lemma \ref{lem:cons_lambda}, $\lambda_a = \Lambda$. 
\end{proof}

\begin{lem}\label{lem:loss_bound_clever}
For all $\D_a \in \A$ and $b \in [0,1]^d$, $d(\lambda_a - \lambda_b) \leq \frac{1}{2}\sum_1^d |a_i - b_i|.$ 
\end{lem}

\begin{proof}
We have two cases.
\paragraph{Case 1:} $b \in [\frac{1}{2} - \Delta, \frac{1}{2} + \Delta]^d$.

Observe that $\lambda_b = \frac{r}{R}f_q(b)$ and $\lambda_a = \frac{r}{R}f_q(a)$. By Lemma \ref{lem:dd}, we see that $f_q$ is $\frac{1}{d^2\sqrt{d}}$-Lipschitz over the domain $[\frac{1}{2} - \Delta, \frac{1}{2} + \Delta]^d$. It follows that
\begin{equation*}
\begin{split}
\lambda_a - \lambda_b &= \frac{r}{R}(f_q(a) - f_q(b)) \\
&\leq \frac{r}{R}||a - b||_2\frac{1}{d^2\sqrt{d}} \\
&= \frac{2\sqrt{d}}{9d^{1/q}}||a - b||_2\frac{1}{d^2\sqrt{d}} \\
&< \frac{||a - b||_1}{2d},
\end{split}
\end{equation*}
with the last inequality following since the $\ell_2$ norm is smaller than the $\ell_1$ norm. Rearranging this gives the statement of the Lemma as desired.

\paragraph{Case 2: } $b \notin [\frac{1}{2} - \Delta, \frac{1}{2} + \Delta]^d$.

The main idea in this case will be to find $b' \in [\frac{1}{2} - \Delta, \frac{1}{2} + \Delta]^d$ such that $\lambda_b \geq \lambda_{b'}$ and such that $||b' - a||_1 \leq ||b - a||_1$. We will then apply Case 1 to get the desired result.

Without loss of generality, assume that $b_1 \geq b_2 \geq \dots \geq b_d$, and that $b_1, b_2, \dots b_k > \frac{1}{2} + \Delta$, $b_{k+1}, \dots, b_l \in [\frac{1}{2} - \Delta, \frac{1}{2} + \Delta]$, and $b_{l+1}, \dots b_d < \frac{1}{2} - \Delta$ for some values of $k$ and $l$. 

We will construct $b'$ in four steps. In each of these steps, we will change the values of $b_i$ such that neither $||a - b||_1$ nor $\lambda_b$ are increased. At each step, we let $b_i$ refer to its value at the end of the previous step.

Finally, for reference, recall that $$\lambda_b = \frac{r}{R}f_q(b) = \frac{r}{R}\sqrt[q]{\sum_1^d |\frac{1}{\sqrt{d}} + \overline{b} - b_i|^q}.$$

\paragraph{Step 1:} We set $$b_i \leftarrow  \begin{cases} \frac{1}{k}\sum_{j=1}^k b_j & 1 \leq i \leq k \\b_i & k+1 \leq i \leq l \\  \frac{1}{d-l}\sum_{j=l+1}^d b_j& l+1 \leq i \leq d \end{cases}.$$ Since $a \in [\frac{1}{2} - \Delta, \frac{1}{2} + \Delta]^d$, we see that these operations do not change $||a - b||_1$, as $\sum_1^k |b_i - a_i| = \sum_1^k b_i - a_i$ and $\sum_{l+1}^d |b_i - a_i| = \sum_1^k a_i - b_i$. Also, observe that this operation preserves $\overline{b}$, and consequently since the function $f(x) = |\frac{1}{\sqrt{d}} + \overline{b} - x|^q$ is convex, we see that by Jensen's inequality that $\lambda_b$ is not increased by this operation.

\paragraph{Step 2:} Let $\beta = \sum_1^k(b_i - \frac{1}{2} - \Delta) - \sum_{l+1}^d (\frac{1}{2} - \Delta - b_i)$. Then we set $$b_i \leftarrow  \begin{cases} \begin{cases} \frac{1}{2} + \Delta + \frac{\beta}{k} & 1 \leq i \leq k \\ b_i & k+1 \leq i \leq l \\ \frac{1}{2} - \Delta & l+1 \leq i \leq d \end{cases} & \beta \geq 0 \\\begin{cases} \frac{1}{2} + \Delta & 1 \leq i \leq k \\ b_i & k+1 \leq i \leq l \\ \frac{1}{2} - \Delta + \frac{\beta}{d - l} & l+1 \leq i \leq d \end{cases} & \beta < 0\end{cases}.$$

Observe that this operation cannot increase $||a - b||_1$, since it doesn't increase $|a_i - b_i|$ for any value of $i$. Furthermore, this operation also does not change $\overline{b}$, and a similar convexity argument on the function $f(x) = |\frac{1}{\sqrt{d}} + \overline{b} - x|^q$ can show that this does not increase $\lambda_b$. 

Finally, if $\beta = 0$, we set $b' = b$, since we have reached a state such that $b \in [\frac{1}{2} - \Delta, \frac{1}{2} + \Delta]^d$. 

\paragraph{Step 3a:} We run this step if $\beta > 0$. Let $\alpha = \frac{\sum_{k+1}^d (\frac{1}{2} + \Delta - b_i)}{\beta}$. We then set  $$b_i \leftarrow  \begin{cases} \begin{cases} \frac{1}{2} + \Delta & 1 \leq i \leq k \\ (\frac{1}{2} + \Delta)(\frac{\alpha - 1}{\alpha}) + \frac{b_i}{\alpha} & k+1 \leq i \leq d \end{cases} & \alpha \geq 1 \\\begin{cases} \frac{1}{2} + \Delta + \frac{\beta}{k}(1 - \alpha) & 1 \leq i \leq k \\ \frac{1}{2} + \Delta & k+1 \leq i \leq d \end{cases} & \alpha < 1\end{cases}.$$ In this step, we can similarly verify that $||a - b||_1$ does not increase (as $|a_i - b_i|$ is strictly reduced for $1 \leq i \leq k$ by an exact amount to offset the possible increases in $|a_i - b_i|$ for $k+1 \leq i \leq d$). We also see by the same convexity argument as usual that this operation reduces $\lambda_b$. 

\paragraph{Step 3b:} We run this step if $\beta < 0$. Let $\alpha = \frac{\sum_{k+1}^d (b_i - \frac{1}{2} + \Delta)}{\beta}$. We then set  $$b_i \leftarrow  \begin{cases} \begin{cases} (\frac{1}{2} - \Delta)(\frac{\alpha - 1}{\alpha}) + \frac{b_i}{
\alpha} & 1 \leq i \leq l \\ \frac{1}{2} - \Delta & k+1 \leq i \leq d \end{cases} & \alpha \geq 1 \\\begin{cases} \frac{1}{2} - \Delta & 1 \leq i \leq l \\ \frac{1}{2} - \Delta + \frac{\beta}{d-l}(1-\alpha) & l+1 \leq i \leq d \end{cases} & \alpha < 1\end{cases}.$$ The justification for this step is analogous to 3a.

\paragraph{Step 4:} We only run this step if $\alpha < 1$. Observe that if $\alpha \geq 1$, then both Step 3a and Step 3b result with $b \in [\frac{1}{2} - \Delta, \frac{1}{2} + \Delta]^d$, which we set as $b'$. Observe that in this case, either $b_i \geq a_i$ for all $i$, or $b_i \leq a_i$ for all $i$. Thus we simply set $$b_i \leftarrow \overline{b}.$$ This operation does not change $||a - b||_1$, and it also reduces $\lambda_b$ (by a convexity argument). 

\paragraph{Step 5:} Finally, for all $1 \leq i \leq d\Delta$, we set $b_i = \frac{1}{2}-\Delta$ if $\overline{b} < \frac{1}{2} - \Delta$ and otherwise set $b_i = \frac{1}{2} - \Delta$ if $\overline{b} > \frac{1}{2} + \Delta$. In both cases, $\lambda_b$ is not changed, and $||a-b||_1$ is strictly reduced. In this step, we finally set $b' = b$. Note that we do not always reach this step, as it was possible in any of the previous steps to reach some $b \in [\frac{1}{2} - \Delta, \frac{1}{2} + \Delta]^d$, at which point we would have simply terminated.

\paragraph{Conclusion: } Through steps $1$ through $5$, we have found $b' \in [\frac{1}{2} - \Delta, \frac{1}{2} + \Delta]^d$ such that $\lambda_{b'} \leq \lambda_b$ and $||a - b'||_1 \leq ||a - b||_1$. By applying Case 1 to $b'$, we see that $d(\lambda_a - \lambda_{b'}) \leq \frac{1}{2}||a - b'||_1$. Thus, we have that $$\frac{1}{2}||a - b||_1 \geq \frac{1}{2}||a - b'||_1 \geq d(\lambda_a - \lambda_{b'}) \geq d(\lambda_a - \lambda_b),$$ which implies the result by the transitive property.

\end{proof}

From the previous two lemmas, we immediately have the following:
\begin{cor}\label{cor:l_1distancebound}
For all $\D_a \in \A$ and $b \in [0,1]^d$, $$\L_r(f_{w^b, 1}, \D_a) \geq \frac{1}{2d}\sum_1^d |a_i - b_i|.$$
\end{cor}

\begin{proof}
We have that
\begin{equation*}
\begin{split}
\L_r(f_{w^b, 1}, \D_a) &\numgeq{a} \frac{d(\lambda_b - \lambda_a) + \sum_1^d|a_i - b_i|}{d - 2d\Lambda} \\
&\geq \frac{\sum_1^d|a_i - b_i| - d(\lambda_a - \lambda_b) + }{d} \\
&\numgeq{b} \frac{\sum_1^d|a_i - b_i| - \frac{1}{2}\sum_1^d |a_i - b_i|}{d} \\
&= \frac{1}{2d}\sum_1^d |a_i - b_i|,
\end{split}
\end{equation*}
where (a) holds by Lemma \ref{lem:loss_bound_general} and (b) holds by Lemma \ref{lem:loss_bound_clever}. 
\end{proof}

\subsubsection{Computing the posterior distribution, $\A|S$}\label{subsubsec:posterior}

Recall that our ultimate goal is to show that $$\E_{\D \sim \A}[\E_{S \sim \D^n}[ \L_r(A_S, \D)]] \geq \Omega(\frac{d}{n}),$$ where $A$ denotes any learning algorithm returning a linear classifier.  The main idea for showing this is to ``switch expectations" and realize that $$\E_{\D \sim \A}[\E_{S \sim \D^n} [\L_r(A_S, \D)]] = \E_{S \sim \B}[\E_{\D \sim \A|S}[\L_r(A_S, \D)]],$$ where $\A|S$ denotes the posterior distribution over $\A$ after observing $S$. In this section, we fully characterize the distribution $\A|S$, and prove several important properties about it.

Recall (Definition \ref{defn:A}) that $\D_a \sim \A$ is generated by first choosing $t_1, t_2, \dots, t_{d/3} \sim \U[0, \frac{\dd}{3}]$ i.i.d, and then letting $a = (a_1, a_2, \dots, a_d)$ be a function of $t = (t_1, \dots, t_{d/3})$. Thus, to compute the posterior $\A|S$, it suffices to focus on the posterior distribution of $t|S$ for any $1 \leq i \leq \frac{d}{3}$. We begin by first defining the likelihood of observing $S$ given that it is generated from parameter $t$.

\begin{defn}\label{defn:L(S|t)}
Let $S = \{(x_1, y_1), (x_2, y_2), \dots, (x_n, y_n)\}$ be any set of $n$ points in $\R^d \times \{\pm 1\}$, and let $t \in [0, \frac{\Delta}{3}]^{d/3}$ be a vector. Let $a \in [\frac{1}{2} - \Delta, \frac{1}{2} + \Delta]^d$ be defined as in Definition \ref{defn:A}. That is, let 
\begin{itemize}
	\item $a_i = \frac{1}{2} + t_i$.
	\item $a_{i + d/3} = \frac{1}{2} + \frac{2\Delta}{3} - g_1(t_i)$.
	\item $a_{i + 2d/3} = \frac{1}{2} - \frac{2\Delta}{3} - g_2(t_i)$. 
\end{itemize}
Then we define $L(S|t)$ as the likelihood of observing the set $S$ from $\D_a^n$. In particular, for any measurable region of points $R \subseteq (\R^d \times \{\pm 1\})^n$, we have that $$\mathbb{P}_{S \sim \D_a^n}[S \in R] = \int_{x \in R}L(x|t)dx.$$
\end{defn}

\begin{lem}\label{lem:binary}
Let $S \subset \R^d \times \{\pm 1\}$ be a set with $n$ points. Then for all $t \in [0, \frac{\dd}{3}]^{d/3}$, $$L(S|t) \in \left\{0, \left(\frac{1}{(d - 2\Lambda)||u||_2}\right)^n\right\},$$ where $\Lambda$ is as defined in Definition \ref{defn:big_lambda} and $L(S|t)$ is as defined in Definition \ref{defn:L(S|t)}. 
\end{lem}

\begin{proof}
Let $\D_a$ be an arbitrary distribution in $\A$. Observe that $\D_a$ is uniform over the set of all points in its support. Thus for every point in its support, we have that the likelihood $L(x|t)$ satisfies $L(x|t) = \frac{1}{(d - 2\Lambda)||u||_2}$. 

Taking the product of this over all points in $S$, we get the desired result. Note that if $S$ contains some point not in the support of $\D_a$, then the likelihood becomes $0$, since the likelihood of observing some point not in the support of $\D_a$ is $0$.
\end{proof}

\begin{defn}\label{defn:permissible_set}
For any dataset $S$, let $P_S$ denote the set of all ``permissible" $t$, that is $t \in [0, \frac{\dd}{3}]^d$ such that $L(S|t) \neq 0$. Formally, $$P_S = \{t: L(S|t) >0\}.$$
\end{defn}

We now fully characterize $P_S$ when $S$ is drawn from some $\D \sim \A$.

\begin{lem}\label{lem:intervals}
Fix $n > 0$. For all $\D \sim \A$ and $S \sim \D^n$, there exist intervals (possibly open, closed, half open) $I_1^S, I_2^S, \dots, I_{d/3}^S \subseteq [0, \frac{\dd}{3}]$ such that $P_S = \prod_1^{d/3} I_i^S$.
\end{lem}

\begin{proof}
Let $S = \{(x_1, y_1), (x_2, y_2), \dots, (x_n, y_n)\}$. Since $S \sim \D^n$, we see that for $1 \leq j \leq n$, $x_j$ must satisfy $x_j \in [v_i, v_i + u]$ for some $1 \leq j \leq d$. Using this, for $1 \leq i \leq d$ let $$s_i^- = \argmax_{\{x_j: x_j \in [v_i, v_i + u], y_j = -1\}} ||x_j - v_i||_2,$$ and $$s_i^+ = \argmax_{\{x_j: x_j \in [v_i, v_i + u], y_j = +1\}} ||x_j - (v_i+ u)||_2.$$ $s_i^-$ and $s_i^+$ can be thought of as the points from $S$ on segment $[v_i, v_i + u]$ that are closest to each other and labeled as $-$ and $+$ respectively. As a default, if no such points exist, we set $s_i^- = v_i$ and $s_i^+ = v_i + u$. 

Next, consider any $t \in [0, \frac{\Delta}{3}]^{d/3}$, let $a \in [\frac{1}{2} - \Delta, \frac{1}{2} + \Delta]^d$ be defined as in Definition \ref{defn:A}. That is, let 
\begin{itemize}
	\item $a_i = \frac{1}{2} + t_i$.
	\item $a_{i + d/3} = \frac{1}{2} + \frac{2\Delta}{3} - g_1(t_i)$.
	\item $a_{i + 2d/3} = \frac{1}{2} - \frac{2\Delta}{3} - g_2(t_i)$. 
\end{itemize}
The key idea of this lemma is that $t \in P_S$ (i.e. $L(S|t) > 0$) if and only if for all $1 \leq i \leq d$, $$[v_i + (a_i - \Lambda)u, v_i + (a_i + \Lambda)u] \subseteq (s_i^-, s_i^+).$$  To see this, observe that if the claim above holds, then we must have that $s_i^- \in [v_i, v_i + (a_i - \Lambda)u)$ and $s_i^+ \in (v_i + (a_i + \Lambda)u, v_i + u]$, and it consequently follows that all points in $S$ are elements of the support of $\D_a$ (Definition \ref{def:w_dist}), as all other points in $S$ are ``further" from the interval $[v_i + (a_i - \Lambda)u, v_i + (a_i + \Lambda)u]$ than the points $s_i^+$ and $s_i^-$. Conversely, if $L(S|t) > 0$, we must have that $S \subseteq supp(\D_a)$, which immediately translates to the statement above. Thus, it suffices to find all $t$ such that this condition holds.

To do this, observe that the interval $[v_i + (a_i - \Lambda)u, v_i + (a_i + \Lambda)u]$ is a line segment of length $2\Lambda||u||_2$ that is centered at the point $v_i + a_iu$. Thus, in order for this to be a sub-segment of $(s_i^-, s_i^+)$, we only need that $a_i$ satisfy $v_i + a_iu \in (s_i^- + \Lambda u, s_i^- - \Lambda u)$. This condition is equivalent to the condition that $a_i \in J_i^S$ for some open interval $J_i^S \subseteq [0, 1]$, where $J_i^S$ is only dependent on $s_i^-, s_i^+$ and $\Lambda$ (which is a constant). In summary, there exist interval $J_1^S, J_2^S, \dots, J_d^S$ such that $t \in P_S$ if and only if $a_i \in J_i^S$ for $1 \leq i \leq d$.

Finally, note that for $1 \leq i \leq d/3$, $a_i, a_{i+d/3}, a_{i + 2d/3}$ are all functions of $t_i$, and moreover these functions are $1$-lipschitz, and monotonic. As a consequence, by taking the intersections of the pre-images of these functions, we find that this condition holds if and only if $t_i \in I_i^S$ where $I_i^S$ is some interval that is a subset of $[0, \frac{\Delta}{3}]^{d/3}$. This proves the claim.
\end{proof}

\begin{cor}\label{cor:posterior}
For any $S \sim \D$ where $\D \sim \A$, let $I_i^S$ be defined as in Lemma \ref{lem:intervals} for $1 \leq i \leq d/3$. Then the posterior distribution $t|S$ is equal to the uniform distribution over the set $\prod_{1 \leq i \leq d/3} I_i^S$, where $t_i$ is sampled from $I_i^S$. 
\end{cor}

\begin{proof}
First, recall that our prior on $t$ is $\U([0, \frac{\Delta}{3}]^d)$, where $\U$ denotes the uniform distribution. By Lemma \ref{lem:binary}, we see that for all $t \in P_S$, $L(S|t) = \left(\frac{1}{(d - 2\Lambda)||u||_2}\right)^n$, and for all other $t$, $L(S|t) = 0$. Furthermore, by Lemma \ref{lem:intervals}, we see that $P_S = \prod_1^{1 \leq i \leq d/3} I_i^S$. Thus, applying Bayes rules gives the desired result. 
\end{proof}

We conclude this section by lower bounding the expected length of the interval $I_i^S$, denoted $\ell(I_i^S)$. 
\begin{lem}\label{lem:expected_length}
For an interval $(c, d) \subset \R$, we let its length, denoted $\ell((c,d))$ be defined as $\ell((c,d)) = d - c$. Then for $1 \leq k \leq d/3$, the expected length (taken over $\D_a \sim \A$ and $S \sim \D_a^n$) of the interval $I_k^S$ is at least $\Omega(\frac{d}{n})$. That is, $$\E_{\D_a \sim \A}\E_{S \sim \D_a^n}[\ell(I_k^S)]] \geq \Omega(\frac{d}{n}).$$
\end{lem}

\begin{proof}
Fix any $\D_{a^*} \sim \Pi$, and let $t^*$ denote the value of $t$ used to generate $a$ (as in Definition \ref{defn:A}). We will show that $\E_{S \sim \D_{a^*}^n}[\ell(I_k^S)]] \geq \Omega(\frac{d}{n}),$ for all $1 \leq k \leq d/3$. We begin by explicitly computing the interval $I_k^S$. 

Fix $1 \leq k \leq d/3$. Then $t_k* \in [0, \frac{\Delta}{3}]$. Assume that $t_k^* > 0$; we will handle the case $t_k^* = 0$ separately. Recall from the proof of Lemma \ref{lem:intervals} that for $1 \leq i \leq d$, we defined $$s_i^- = \argmax_{\{x_j: x_j \in [v_i, v_i + u], y_j = -1\}} ||x_j - v_i||_2,$$ and $$s_i^+ = \argmax_{\{x_j: x_j \in [v_i, v_i + u], y_j = +1\}} ||x_j - (v_i+ u)||_2.$$ for $1 \leq i \leq d$. 

Next let $t \in [0, \frac{\Delta}{3}]^{d/3}$ be a vector, and let $a \in [\frac{1}{2} - \Delta, \frac{1}{2} + \Delta]^d$ be defined as $a_k = \frac{1}{2} + t_k$, $a_{k + d/3} = \frac{1}{2} + \frac{2\Delta}{3} - g_1(t_k)$ and $a_{k + 2d/3} = \frac{1}{2} - \frac{2\Delta}{3} - g_2(t_k)$, for $1 \leq k \leq d/3$. Note that $g_1, g_2$ are the functions defined in Definition \ref{defn:g_1_and_g_2}.

As we argued in the proof of Lemma \ref{lem:intervals}, it then follows that $t_k \in I_k^S$ if and only if $$[v_i + (a_i - \Lambda)u, v_i + (a_i + \Lambda)u] \subseteq (s_i^-, s_i^+),$$ for $i = k, k+d/3, k+2d/3$. Finally, as we did in Lemma \ref{lem:intervals}, for each $1 \leq i \leq d$, we define intervals $J_i^S \subseteq [\frac{1}{2} - \Delta, \frac{1}{2} + \Delta]$ such that $a_i \in J_i^S$ if and only if $[v_i + (a_i - \Lambda)u, v_i + (a_i + \Lambda)u] \subseteq (s_i^-, s_i^+)$.

We now have the following three claims.

\paragraph{Claim 1:} Let $\alpha = \min \left(\frac{||s_k^- -  (v_k + (a_k^* - \Lambda)u)||_2}{||u||_2}, t_k^*\right)$. If $t_k \in (t_k^* - \alpha, t_k^*]$, then $$[v_k + (a_k - \Lambda)u, v_k + (a_k + \Lambda)u] \subseteq (s_k^-, s_k^+).$$ 

\textit{Proof: } First, observe that since $s_k^+$ and $s_k^-$ were sampled from $\D_{a^*}$, it follows that $$[v_k + (a_k^* - \Lambda)u, v_k + (a_k^* + \Lambda)u] \subseteq (s_i^-, s_i^+).$$ Consider any $t_k \in [t_k^* - \alpha, t_k^*]$. Then substituting the definitions of $a_k, a_k^*$ imply that $a_k \in [a_k^* - \alpha, a_k^*]$. Because of this, it follows that 
\begin{equation*}
\begin{split}
||(v_k + (a_k - \Lambda)u) - (v_k + (a_k^* - \Lambda)u)||_2 &= ||(a_k - a_k^*)u||_2 \\
&< \alpha||u||_2 \\
&\leq ||s_k^- - (v_k + (a_k^* - \Lambda)u)||_2,
\end{split}
\end{equation*}
which implies that $v_k + (a_k - \Lambda)u \in (s_i^-, v_k + (a_k^* - \Lambda)u]$. Furthermore, the fact that $a_k \leq a_k^*$ implies that $v_k + (a_k + \Lambda)u \in (v_k + (a_k - \Lambda)u, v_k + (a_k^* + \Lambda)u]$. 

Together, these observations imply the desired result, as it follows that $$[v_k + (a_k - \Lambda)u, v_k + (a_k + \Lambda)u] \subset (s_k^-, v_k + (a_k^* + \Lambda)u] \subset (s_k^-, s_k^+).$$ $\blacksquare$

\paragraph{Claim 2:} Let $\beta = \min \left(\frac{||s_{k+d/3}^+ -  (v_{k+d/3} + (a_{k+d/3}^* + \Lambda)u)||_2}{||u||_2}, g_1(t_{k}^*)\right)$. If $t_k \in (g_1^{-1}(g_1(t_k^*) - \beta), t_k^*]$, then $$[v_{k+d/3} + (a_{k+d/3} - \Lambda)u, v_k + (a_{k+d/3} + \Lambda)u] \subseteq (s_{k+d/3}^-, s_{k+d/3}^+).$$ 

\textit{Proof: } First, we observe that $\beta$ is well defined since $g_1$ is a monotonic $1$-Lipschitz function, and consequently has an inverse. Next, we also see that $0 \leq g_1(t_k^*) - g_1(t_k) \leq \beta$. Substituting the definitions of $a_k^*, a_k$, it follows that $0 \leq a_k - a_k^* \leq \beta$ (notice the order switch). At this point, we can apply the same argument as in Claim 1 to get the desired result.  $\blacksquare$.

\paragraph{Claim 3:} Let $\tau = \min \left(\frac{||s_{k+2d/3}^+ -  (v_{k+2d/3} + (a_{k+2d/3}^* + \Lambda)u)||_2}{||u||_2}, g_2(t_k^*)\right)$. If $t_k \in (g_2^{-1}(g_2(t_k^*) - \tau), t_k^*]$, then $$[v_{k+2d/3} + (a_{k+2d/3} - \Lambda)u, v_{k+2d/3} + (a_{k+2d/3} + \Lambda)u] \subseteq (s_{k+2d/3}^-, s_{k+2d/3}^+).$$  

\textit{Proof: }  Completely analogous to Claim 2. $\blacksquare$.

Combining these claims, we see that if $t_k \in (t_k^* - \alpha, t_k^*] \cap (g_1^{-1}(g_1(t_k^*) - \beta), t_k^*] \cap  (g_2^{-1}(g_2(t_k^*) - \tau), t_k^*]$, then $t_k \in I_k^S$. Since these three intervals all have an endpoint in $t_k^*$, it follows that there is an interval with length $\eta$ that is a subset of $I_k^S$, where $$\eta = \min(\ell((t_k^* - \alpha, t_k^*])), \ell((g_1^{-1}(g_1(t_k^*) - \beta), t_k^*]), \ell((g_2^{-1}(g_2(t_k^*) - \tau), t_k^*])).$$ However, by substituting that $g_1, g_2$ are $1$-Lipschitz, we see that $\ell((g_1^{-1}(g_1(t_k^*) - \beta), t_k^*]) \geq \beta$ and $\ell((g_2^{-1}(g_2(t_k^*) - \tau), t_k^*])) \geq \tau$. Thus, it follows that $$\ell(I_k^S) \geq \eta \geq \min(\alpha, \beta, \tau).$$ Thus it suffices to show that $\E_{S \sim \D_{a^*}}[\min(\alpha, \beta, \tau)] \geq \Omega(\frac{d}{n})$. 

To do this, observe that
\begin{itemize}
	\item $\alpha||u||_2$ is the distance from the closest point labeled $-$ on the segment $[v_k, v_k + u]$ to the point $v_k + (a_k^* - \Lambda)u$
	\item  $\beta||u||_2$ is the distance from the closest point labeled $+$ on the segment $[v_{k+d/3}, v_{k+d/3} + u]$ to the point $v_{k+d/3} + (\Lambda + a_{k+d/3}^*)u$
	\item $\tau||u||_2$is the distance from the closest point labeled $+$ on the segment $[v_{k + 2d/3}, v_{k + 2d/3} + u]$ to the point $v_{k+2d/3} + (\Lambda + a_{k+2d/3}^*)u$.
\end{itemize}

Finally, it is not difficult to see that for sufficiently large $n$, with high probability each of these distances will be $\Omega(\frac{d}{n})$. This is because with high probability there will be $\Theta(\frac{n}{d})$ points on each of the respective line segments, and we are considering the closest point among them to some reference point. Thus, it follows that with high probability $\E_{S \sim \D_{a^*}}[\min(\alpha, \beta, tau)] \geq \Omega(\frac{d}{n}),$ as desired.
\end{proof}

\subsubsection{Putting it all together, the proof}\label{subsubsec:proof}

We prove the following key lemma, which directly implies Theorem \ref{thm:lower}.

\begin{lem}\label{lem:lower_bound}
Let $M$ be any learning algorithm that outputs a linear classifier. For any training sample of points $S = \{(x_1, y_1), (x_2, y_2), \dots, (x_n, y_n)\}$, we let $M_S$ denote the classifier learned by $M$ from $S \sim \D$. Then it follows that $$\E_{\D \sim \A} \E_{S \sim \D^n}[\L_r(M_S, \D)]] \geq \Omega(\frac{d}{n}).$$ 
\end{lem}

\begin{proof}
Let $\F_n$ denote the distribution over $(\R^d \times \{\pm 1\})^n$ defined as the composition $\D \sim \A$ and $S \sim \D^n$. That is, $S \sim \F_n$ follows the same distribution as $\D \sim \A, S \sim \D^n$. Then we can write the expectation above as 
\begin{equation*}
\begin{split}
\E_{\D \sim \A} \E_{S \sim \D^n}[\L_r(A_S, \D)]] = \E_{S \sim \F_n} \E_{\D \sim (\A|S)}[\L_r(M_S, \D)]],
\end{split}
\end{equation*}
where $\A|S$ denotes the posterior distribution of $\D$ conditioned on observing $S$. First, fix any such $S$. We will bound $\E_{\D \sim (\A|S)}[\L_r(M_S, \D)].$ First, by reparametrizing in terms of $t \in [0,\frac{\dd}{3}]^{d/3}$ and applying Corollary \ref{cor:posterior}, we have that $$\E_{D \sim (\A|S)}[\L_r(M_S, \D)] = \E_{t_1 \sim \U(I_1^S)}[\dots [\E_{t_n \sim \U(I_{d/3})}[\L_r(M_S, \D_a)]\dots ],$$ where $I_1^S, I_2^S, \dots, I_{d/3}^S \subset [0, \frac{\dd}{3}]$ are the intervals defined in Lemma \ref{lem:intervals}, and $a$ is defined as in Definition \ref{defn:A}. 

Next, let $b \in [0, 1]^d$ be such that $M_S = f_{w^b, 1}$, where $w^b$ is defined as in Definition \ref{def:normal_vector}. Then it follows from Corollary \ref{cor:l_1distancebound} that 
\begin{equation*}
\begin{split}
\L_r(M_S, \D_a)] &\geq \frac{1}{20d}\sum_1^d |a_i - b_i| \\
&\geq \frac{1}{20d}\sum_1^{d/3} |\frac{1}{2} + t_i - b_i|
\end{split}
\end{equation*}
with the last inequality coming from substituting the definition of $a_i$ and (and ignoring $a_i$ for $i > d/3$). We now take the expectation of this inequality over $t_1, t_2, \dots, t_{d/3}$. To do so, observe that by simple algebra, $\E_{t_i \sim \U(I_i^S)} |\frac{1}{2} + t_i - b_i| \geq \frac{\ell(I_i^S)}{4}$. Substituting this, we see that $$E_{t_1 \sim \U(I_1^S)}[\dots [\E_{t_n \sim \U(I_{d/3}^S)}[\L_r(M_S, \D_a)]\dots ] \geq \frac{1}{80d} \sum_{i=1}^{d/3} \ell(I_i^S).$$ Finally, by taking expectations over $S \sim \F_n$, we see that 
\begin{equation*}
\begin{split}
\E_{\D \sim \A} \E_{S \sim \D^n}[\L_r(A_S, \D)]] &= \E_{S \sim \F_n} \E_{\D \sim (\A|S)}[\L_r(M_S, \D)]] \\
&\geq \E_{S \sim \F_n} \frac{1}{80d} \sum_{i=1}^{d/3} \ell(I_i^S) \\
&= \frac{1}{80d}\sum_1^{d/3}\E_{S \sim \F}[\ell(I_i^S)] \\
&= \frac{1}{80d}\sum_1^{d/3}\E_{\D \sim \A}\E_{S \sim \D^n}[\ell(I_i^S)] \\
&\geq \frac{1}{80d}\sum_1^{d/3}\Omega(\frac{d}{n}) = \Omega(\frac{d}{n}),
\end{split}
\end{equation*}
where the last step follows from Lemma \ref{lem:expected_length}. 
\end{proof}

Finally, we can prove Theorem \ref{thm:lower}.

\begin{proof}
(Theorem \ref{thm:lower}). First, by Lemmas \ref{lem:separation} and \ref{lem:large_margin}, we see that $\A \subseteq \F_{r, \rho}$ (provided $\rho > 10$). Next, by Lemma \ref{lem:lower_bound}, for any $n$ there must exists some $\D \sim \A$ such that $\E_{S \sim \D^n}[\L_r(M_S, \D)] \geq \Omega(\frac{d}{n})$. Thus selecting this distribution suffices. This concludes the proof.
\end{proof}

\section{Proofs for Algorithm \ref{alg:upper_bound}}\label{sec:upper_bound_details}

This section is divided into 2 parts. In section \ref{sec:upper_bound_origin}, we show that for the case in which our data distribution $\D$ is linearly $r$-separated by some hyperplane through the origin, the desired error bound holds. That is, we prove Theorem \ref{thm:upper_bound} under this assumption.

Next, in section \ref{sec:upper_bound_general}, we show how to generalize Algorithm \ref{alg:upper_bound} to arbitrary linearly $r$-separated distributions, and subsequently prove Theorem \ref{thm:upper_bound} in the general case.

\subsection{Origin Case}\label{sec:upper_bound_origin}

We begin by precisely stating the conditions required in the ``origin" case. We assume the following properties hold for our data distribution $\D$. We let $S_r^+$ and $S_r^-$ be defined as in section \ref{sec:upper_bound}.

\begin{enumerate}
	\item There exists $R > 0$ such that for all $x \in S_r^+ \cup S_r^-$, $||x||_2 \leq R$.
	\item There exists a unit vector $u \in \R^d$ and $\gamma_r > 0$ such that 
	\begin{itemize}
		\item $\L_r(f_{u, 0}, \D) = 0$, where $f_{u, 0}$ denotes the linear classifier with decision boundary $\langle u, x \rangle = 0$. 
		\item $S_r^+ \cup S_r^-$ has distance at least $\gamma_r$ from the decision boundary of $f_w$. That is, $||S_r^+ \cup S_r^- - H_{u, 0}||_2 \geq \gamma_r$.
	\end{itemize}
	\item By the previous conditions, it follows that $\langle u, yx' \rangle \geq \gamma_r$ for all $(x,y) \sim \D$, and $x' \in B_p(x, r)$. This is because $u$ is a unit vector. 
\end{enumerate}

Next, before analyzing Algorithm \ref{alg:upper_bound}, we will first give a slight modification of the algorithm that lends itself to better analysis. The only difference is that in this new algorithm, we first randomly sample $k \sim \{1, 2, \dots, n\}$, and then only train on the first $r$ data-points of our training sample.
\begin{algorithm}[H]
   \caption{Modified-Adversarial-Perceptron}
   \label{alg:upper_bound_modified}
\begin{algorithmic}[1]
    \STATE \textbf{Input}:  $S = \{(x_1, y_1), \dots, (x_n, y_n)\} \sim \D^n,$
    \STATE $w \leftarrow 0$ 
    \STATE $k \sim \U(\{0, 1, 2, \dots, n\})$
    \FOR{$i = 1 \dots k$}
    	\STATE $z = \argmin_{||z - x_i||_p \leq r}  y_i\langle w, z \rangle$ 
        \IF{$\langle w, y_iz \rangle \leq 0$}
            \STATE $w \leftarrow w + y_iz$
        \ENDIF           
    \ENDFOR
    \STATE return $f_{w, 0}$
\end{algorithmic}
\end{algorithm}

We will show that Algorithm \ref{alg:upper_bound_modified} satisfies the guarantees of Theorem \ref{thm:upper_bound_origin}. We begin with the following, key lemma.

\begin{lem}\label{lem:update_count}
Under the assumptions above about $\D$, Algorithm \ref{alg:upper_bound_modified} makes at most $\frac{R^2}{\gamma_r^2}$ updates to $w$.
\end{lem}

\begin{proof}
Let $w_t$ denote our weight vector after we make $t$ updates. Observe that $w_t = w_{t-1} + y_tx_t + z'$ where $(x_t, y_t)$ denotes the point we made a mistake on, and $z' = \argmin_{|z|_p \leq r} \langle w, z \rangle$. Letting $x_t' = x_t + y_tz'$, we see that $w_t = w_{t-1} + y_tx_t'$. Now the key observation is that $(x_t', y_t) \in S_r^+ \cup S_r^-$, and as a result, it follows that $\langle u, y_tx_t' \rangle \geq \gamma_r$. Using this, we see that
\begin{equation*}
\begin{split}
\langle u, w_t \rangle &= \langle u, w_{t-1} + y_tx_t' \rangle \\
&= \langle u, w_{t-1} \rangle + \langle u, y_tx_t' \rangle \\
&\geq \langle u, w_{t-1} \rangle + \gamma_r.
\end{split}
\end{equation*}
Thus, by a simple proof by induction, we see that $\langle w_t, u \rangle \geq t\gamma_r$. 

Next, observe that we must have $\langle w_{t-1}, y_tx_t' \rangle \leq 0$. This is because $w_{t-1}$ must missclassify $(x_t', y_t)$ (thus failing to be astute at $(x_t, y_t)$) in order for it to be updated. Substituting this, we see that
\begin{equation*}
\begin{split}
||w_t||_2 &= \sqrt{\langle w_t, w_t \rangle} \\
&= \sqrt{\langle w_{t-1} + x_t'y_t, w_{t-1} + x_t'y_ \rangle} \\
&= \sqrt{\langle w_{t-1}, w_{t-1} \rangle + 2\langle w_{t-1}, x_t'y_t \rangle + \langle x_t', x_t' \rangle} \\
&\leq \sqrt{||w_{t-1}||_2^2 +  0 + R^2},
\end{split}
\end{equation*}
with the last inequality holding since $|x_t'|_2 \leq R$. Thus, by a simple proof by induction, we see that $||w_t||_2 \leq R\sqrt{t}$. 

Finally, since $u$ is a unit vector, it follows that $||w_t||_2 \geq \langle w_t, u$. Substituting our inequalities, we find that $R\sqrt{t} \geq \gamma_r t$ which implies that $t \leq \frac{R^2}{\gamma_r^2}$. Since $t$ is the number of mistakes we make, the result follows. 
\end{proof}

\begin{lem}\label{thm:upper_bound_origin}
Let $\D$ be a distribution with the assumptions above. For any $S \sim \D^n$, let $f_S$ denote the classifier learned by Algorithm \ref{alg:upper_bound_modified}. Then $$\E_{S \sim \D^n}\L_r(f_S, \D) \leq \frac{R^2}{\gamma_r^2 (n+1)}.$$
\end{lem}

This Theorem directly follows from the classic online to offline result (Theorem 3 of \cite{Freund99}). For completeness, we include a proof in our context.

\begin{proof}
Fix any $n$ and consider running Algorithm \ref{alg:upper_bound_modified} on $S \sim \D^n$. Let $L_t$ denote the expected robust loss of our classifier conditioning on $k = t$, and let $L^*$ denote the expected overall loss of our classifier. It follows that $$\E_{S \sim \D^n} L^* = \frac{1}{n+1}\sum_{t=0}^n \E_{S \sim \D^n}[L^*|k = t] = \frac{1}{n+1}\sum_{t=0}^n \E_{S \sim \D^n}[L_t].$$

Next, let $T \sim \D^{n+1}$ be a separate i.i.d drawn sample, and suppose we run the adversarial perceptron algorithm on the entirety of $T$ (i.e. rung Algorithm \ref{alg:upper_bound_modified} on $T$ by setting $k = n+1$). For $1 \leq t \leq n + 1$, let $X_t$ be the indicator variable for whether the $t$th point in $T$ requires an update on $w$ (i.e. the classifier is not astute at $w$). There are two important observations to make.

First, we have that $\E_{T \sim \D^{n+1}}[X_t] = \E_{S \sim \D^n}[L_{t-1}]$. This is because $X_t$ is an indicator variable for a classifier trained on precisely $t-1$ i.i.d training examples lacking astuteness for a randomly drawn point from $\D$. Second, we have that $\sum_{t = 1}^{n+1} X_t \leq \frac{R^2}{\gamma_r^2}$. This is because each $\sum X_t$ is precisely the number of updates that perceptron makes on $T$, which is bounded by Lemma \ref{lem:update_count}. By combining these two observations, we see that 
\begin{equation*}
\begin{split}
\E_{S \sim \D^n}[L^*] &= \frac{1}{n+1}\sum_{t=0}^n \E_{S \sim \D^n}[L_t] \\
&= \frac{1}{n+1}\sum_{t=0}^n \E_{T \sim \D^{n+1}}[X_{t+1}] \\
&= \frac{1}{n+1}\E_{T \sim \D^{n+1}}[\sum_{t = 1}^{n+1} X_{t}] \\
&\leq \frac{R^2}{\gamma_r^2(n+1)},
\end{split}
\end{equation*}
as desired. 
\end{proof}

\subsection{General Case}\label{sec:upper_bound_general}

In general case, we no longer assume that the optimal classifier $f_{u, b}$ passes through the origin. To account for this, we will need to first adapt our algorithm. The basic idea is to simply append a $1$ to the vectors $x$ and increase the dimension $d$ by $1$. We are then left with solving a $d+1$ dimensional problem in which the data is once-again separated by a hyperplane passing through the origin. 

We begin with two useful sets of notation.

\begin{defn}
We use the following notation:
\begin{itemize}
	\item For any $x \in \R^d$ and $R \in \R$, we let $x|R \in \R^{d+1}$ denote the $d+1$ dimensional vector obtained by appending the value $R$ to $x$. 
	\item For $w \in \R^{d+1}$, let $||w||_q^*$ denote the $\ell_q$ norm of the first $d$ coordinates of $w$.
	\item For $x \in \R^{d+1}$, let $B_p^*(x, r)$ denote all $z \in \R^{d+1}$ such that $||z - x||_p \leq r$ and such that $z$ and $x$ both share the same last coordinate.
	\item For $S = \{(x_1, y_1), \dots, (x_n, y_n)\} \subset \R^{d+1} \times \{\pm 1\}$, let $R_S$ denote $\max_{i \neq j} ||x_i - x_j||_2$. 
\end{itemize}

\end{defn}

We now propose the following modified version of Algorithm \ref{alg:upper_bound}, that is capable of handling any dataset, including ones that aren't separated by a hyperplane through the origin.
\begin{algorithm}[H]
    \caption{
        General-Adversarial-Perceptron
    }
    \label{alg:gen_upper_bound}
    
    \begin{algorithmic}[1]
    \STATE \textbf{Input}:  $S = \{(x_1, y_1), \dots, (x_n, y_n)\} \sim \D^n,$
    \STATE $x_i' \leftarrow x_i - x_1$. 
    \STATE $R_S = diam_2(S)$
    \STATE $w \leftarrow 0 \in \R^{d+1}$
    \STATE Randomly permute $S$
    \STATE Randomly choose $k \in \{1, 2, 3, \dots, n\}$. 
    \FOR{$t = 1 \dots k$}   
        \IF{$\langle w, y_t(x_t|R_S) \rangle \leq r||w||_q^*$}
            \STATE $z' = \argmin_{|z|_p \leq r} \langle w, z|0 \rangle$
            \STATE $w \leftarrow w + y_t(x_t|R_S) + z'|0$
        \ENDIF           
    \ENDFOR
    \STATE $w^* \leftarrow$ first $d$ coordinates of $w$
    \STATE $b \leftarrow$ the last element of $w$
    \STATE Return $f_{w^*, \langle w^*, x_1 \rangle -bR_S}$
    \end{algorithmic}
\end{algorithm}

The basic idea of the algorithm is to first translate $S$ so that one point is the origin, and then append $R_S$ to every vector in $S$ so that each vector is now $d+1$ dimensional. After doing this, we apply Algorithm \ref{alg:upper_bound} as before with one important difference: for our adversarial attacks, we make sure to not change the last coordinate. 

We now show that this algorithm has a similar performance to our old algorithm. We first prove a helpful lemma.

\begin{lem}\label{lem:general_upper_bound}
Let $\D$ be any linearly $r$-separated distribution, and let $S \sim \D^n$ such that $S$ has positively and negatively labeled examples. Let $x_i' = x_i - x_1$ for $1 \leq i \leq n$. Then the following hold.
\begin{itemize}
	\item There exists a unit vector $u \in \R^{d+1}$ such that for all $(x_i, y_i) \in S$, $\min_{z \in B_p^*(x_i')} \langle u, y_i(z|R_S) \rangle \geq \frac{\gamma_r(\D)}{\sqrt{2}}.$
	\item For all $(x_i, y_i) \in S$, $||x_i'|R_S||_2 \leq \sqrt{2}diam_2(\D)$. 
\end{itemize}
\end{lem}

\begin{proof}
Without loss of generality, we will assume $x_1 = 0$ so that we can safely ignore the differences between $x_i'$ and $x_i$. Since $\D$ is $r$-separated, there exist $w, b$ (with $w$ a unit vector) such that $$\langle w, zy \rangle \geq by + \gamma_r(\D),$$ for all $(x,y) \sim \D$ and $z \in B_p(x, r)$. Furthermore, since $x_1 = 0$, it follows that $||x||_2 \leq \diam_2(\D)$ for all $(x, y) \sim \D$. This immediately implies that $||x_i|R_S||_2 \leq \sqrt{\diam_2(\D)^2 + R_S^2} \leq \sqrt{2}\diam_2(\D)$, yielding the second part of the lemma.

For the first part, observe that we can rearrange the equation above, we see that $$\langle w|-\frac{b}{R_S}, zy | R_S \rangle \geq \gamma_r(\D).$$ The key observation is that the first equation implies that $b \leq R_S$. This is because $S$ contains positively and negatively labeled examples, and consequently $\langle w, x_i \rangle \geq b + \gamma_r(\D) > b$ for some $x_i$ such that $|x_i| = R_S$. Thus, it follows that the unit vector $u = \frac{w|\frac{-b}{R_S}}{\sqrt{1 + b^2/R_S^2}}$ has the desired property, by observing that $\sqrt{1 + b^2/R_S^2} \leq \sqrt{2}$. 
\end{proof}

Lemma \ref{lem:general_upper_bound} allows us to analyze the performance of Algorithm \ref{alg:gen_upper_bound}. The basic idea is that our performance on the transformed data in $\R^{d+1}$ is isomorphic to its performance on the data in $\R^d$. As a consequence, we can apply the same argument as in Theorem \ref{thm:upper_bound_origin} to get a bound on the error estimate. However, this bound must be given in terms of the diameter and robust margin of the \textit{transformed data}: quantities that have been bounded in Lemma \ref{lem:general_upper_bound}. Thus, putting this all together, Theorem \ref{thm:upper_bound} follows.

\section{Details for Kernel Algorithm}\label{sec:kernel_appendix}

Next, we find analogs of linear $r$-separability and the robust margin when considering kernels. First, we define an embedding function.

\begin{defn}\label{defn:embedding_function}
Let $K: \R^d \times \R^d \to \R^+$ be a kernel similarity function. Then there exists a Hilbert space $H$ and map $\phi: \R^d \to H$ such that for all $x_1, x_2 \in \R^d,$ we have $$K(x_1, x_2) = \langle \phi(x_1), \phi(x_2) \rangle.$$ We call $\phi$ the \textbf{embedding function} and $H$ the \textbf{embedding space}.
\end{defn}

The key idea of this section is that Kenrel classifiers correspond to linear classifiers in embedded space. This is the essence of the ``kernel trick." Formally, we have the following, well-known theorem. 

\begin{thm}\label{thm:kernel_trick}
Let $K:\R^d \times \R^d \to \R^+$ be a kernel similarity function. Let $T = \{(x_1, y_1), \dots, (x_m, y_m)\} \subset \R^d \times \{\pm 1\}$ be a set of labeled points, and $\alpha \in \R^m$ be a vector of $m$ real numbers. Then for all $x \in \R^d$, we have that $$\sum_{i = 1}^m \alpha_iy_iK(x_i, x) = \big \langle \sum_{i= 1}^m \alpha_iy_i\phi(x_i), \phi(x) \big \rangle.$$ Because of this, if we let $w = \sum_{i=1}^m \alpha_iy_i\phi(x_i)$, then the kernel classifier $f_{T, \alpha}^k$ satisfies $f_{T, \alpha}^k(x) = f_{w, 0}(\phi(x))$, where the latter classifier is the linear classifier in $H$ with weight vector $w$. 
\end{thm}

The main idea behind Algorithm \ref{alg:upper_bound_kernel}, is that it corresponds to running Algorithm \ref{alg:upper_bound} inside the embedded space of the kernel $K$. In particular, the kernel-perceptron update step precisely corresponds to the dual-form of the perceptron-update step inside embedded space. It follows from Theorem \ref{thm:kernel_trick} that the following algorithm is identical to Algorithm \ref{alg:upper_bound_kernel}. 

\begin{algorithm}[H]
   \caption{Adversarial-Kernel-Perceptron}
   \label{alg:upper_bound_kernel_nice}
\begin{algorithmic}[1]
    \STATE \textbf{Input}:  $S = \{(x_1, y_1), \dots, (x_n, y_n)\} \sim \D^n,$ Similarity function, $K$,
    \STATE $w \leftarrow 0$
    \FOR{$i = 1 \dots n$}
    	\STATE $z = \argmin_{||z - x||_p \leq r}  y_i\langle w, \phi(z) \rangle$ 
        \IF{$\langle y_iw, \phi(z) \rangle \leq 0$}
            \STATE $w = w + y_i\phi(z)$ 
        \ENDIF           
    \ENDFOR
    \STATE return $f_{w,0} \circ \phi$
\end{algorithmic}
\end{algorithm}
In particular, by comparing Algorithms \ref{alg:upper_bound_kernel} and \ref{alg:upper_bound_kernel_nice}, we have by Theorem \ref{thm:kernel_trick} that for all time steps $t$, $$w = \sum_{(z,y) \in T} y\phi(z).$$ Therefore, to analyze the performance of Algorithm \ref{alg:upper_bound_kernel}, it suffices to analyze Algorithm \ref{alg:upper_bound_kernel_nice}. However, we already have built to the tools for doing this: all of the results from Section \ref{sec:upper_bound_origin} apply to Algorithm \ref{alg:upper_bound_kernel_nice} since the only difference is replacing $\R^d$ with $H$, the embedding space of $K$. 

We now proceed by giving the corresponding assumptions on $\D$ needed for Theorem \ref{thm:upper_bound_kernel}. We begin by first defining $(K, r)$-separability and $K$-robust margin, $\gamma_{r, K}$, the Kernel analogs of linear $r$-separability (Definition \ref{defn:r_separability}) and the robust margin (Definition \ref{def:robust_margin}).

\begin{defn}\label{defn:ker_r_separability}
For any $r > 0$, a distribution $\D$ over $\R^d \times \{\pm 1\}$ is $(K, r)$-\textbf{separable} if there exists a kernel classifier $f_{S, \alpha}^K$ such that $\L_r(f_{S, \alpha}^K, \D) = 0$.
\end{defn}

To define the $K$-robust margin, we will once again need the sets $S_r^+$ and $S_r^-$ defined in equation \ref{eqn:s_plus_s_minus} (top right of page 7). Recall that these sets denote the positively and negatively labeled elements from $supp(\D)$ \textit{including} all adversarial perturbations of those points. 

\begin{defn}\label{defn:k_rob_margin}
Let $\D$ be a $(K, r)$-separable distribution over $\R^d \times \{ \pm 1\}$. Then $\D$ has $K$-robust margin $\gamma_r$ if $\gamma_r$ is the largest real number such that there exists a kernel classifier $f_{T, \alpha}^K$, such that the following conditions hold.

\begin{enumerate}
	\item $\L_r(f_{T, \alpha}^K, \D) = 0$. 
	\item Let $\phi, H$ be the embedding function/space of $K$, let $w = \sum_{(z, y) \in T} y\phi(z)$, and let $H_w = \{z \in H, \langle z, w \rangle = 0\}$ be the decision boundary in $H$ of $f_{T, \alpha}^K$. Then for all $x \in S_r^+ \cup S_r^-$, $\phi(x)$ has $\ell_2$ distance at least $\gamma_r^K$ from $H_w$ inside $H$. That is, $$\inf_{x \in S_r^+ \cup S_r^-} \inf_{z \in H_w} \sqrt{\langle \phi(x) - z, \phi(x) - z \rangle} = \gamma_r^K.$$ 
\end{enumerate}
\end{defn}

We now state the main theorem giving the performance of Algorithm \ref{alg:upper_bound_kernel}. 

\begin{thm}
Let $\D$ be a distribution over $\R^d \times \{\pm 1\}$ such that the following conditions hold. 
\begin{enumerate}
	\item There exists $R > 0$ such that for all $x \in S_r^+ \cup S_r^-$, $\langle \phi(x), \phi(x) \rangle \leq R^2$.
	\item $\D$ is $K, r$-separable, and has $K$-robust margin $\gamma_r^K > 0$.
\end{enumerate}
Then for any $S \sim D^n$, if $f_{T, \alpha}^k$ denotes the classifier learned by Algorithm \ref{alg:upper_bound_kernel}, then $$\E_{S \sim \D^n}[\L_r(f_{T, \alpha}^k, \D)] = O\left(\frac{(\gamma_r^K)^2}{R^2(n+1)} \right).$$
\end{thm}

\begin{proof}
The key idea is to observe that Lemmas \ref{lem:update_count} and \ref{thm:upper_bound_origin} both directly translate from Algorithm \ref{alg:gen_upper_bound} to Algorithm \ref{alg:upper_bound_kernel_nice}. In particular, neither proof used the dimension, $d$, of $\R^d$, and consequently would equally apply to even an infinite dimensional Hilbet Space, $H$. Thus, the proof is completely analogous to the proof of Theorem \ref{thm:upper_bound_origin}.
\end{proof}

\end{document}